\documentclass[twoside]{article}

\usepackage[accepted]{aistats2025}
\usepackage[round]{natbib}
\renewcommand{\bibname}{References}
\renewcommand{\bibsection}{\subsubsection*{\bibname}}
\usepackage{amsmath, amsthm, amssymb}
\usepackage{algorithm}
\usepackage{algpseudocode}
\usepackage{bm}
\usepackage{enumitem}
\usepackage{amssymb}
\usepackage{hyperref}
\usepackage{url}
\usepackage{mathtools,color,cleveref}
\usepackage{dsfont}
\usepackage{tablefootnote}

\DeclareMathOperator*{\argmax}{arg\,max}
\DeclareMathOperator*{\argmin}{arg\,min}

\DeclareMathOperator*{\essinf}{ess\,inf}
\algrenewcommand\algorithmicrequire{\textbf{Input:}}

\newtheorem{theorem}{Theorem}
\newtheorem{lemma}{Lemma}

\newtheorem{remark}{Remark}
\newcommand{\myMatrix}[1]{{#1}}
\newcommand{\vect}[1]{{#1}} 
\newcommand{\nn}{\nonumber}
\newcommand*\varient[1]{\bar{#1}}
\newcommand{\Var}{\mathrm{Var}}
\newcommand{\edited}[1]{}

\begin{document}

%

%

\twocolumn[

\aistatstitle{Near-Optimal Sample Complexity for Iterated CVaR Reinforcement Learning with a Generative Model}

\aistatsauthor{ Zilong Deng \And Simon Khan \And Shaofeng Zou }

\aistatsaddress{ECEE, Arizona State University \And  Air Force Research Laboratory \And ECEE, Arizona State University } ]

\begin{abstract}
  In this work, we study the sample complexity problem of risk-sensitive Reinforcement Learning (RL) with a generative model, where we aim to maximize the Conditional Value at Risk (CVaR) with risk tolerance level $\tau$ at each step, named Iterated CVaR. 
  We develop nearly matching upper and lower bounds on the sample complexity for this problem. Specifically, we first prove that a value iteration-based algorithm, ICVaR-VI, achieves an $\epsilon$-optimal policy with at most $\tilde{\mathcal{O}}\left(\frac{SA}{(1-\gamma)^4\tau^2\epsilon^2}\right)$ samples, where $\gamma$ is the discount factor, and $S, A$ are the sizes of the state and action spaces. Furthermore, 
  if $\tau \geq \gamma$, then the sample complexity can be further improved to $\tilde{\mathcal{O}}\left( \frac{SA}{(1-\gamma)^3\epsilon^2} \right)$. 
  We further show a minimax lower bound of ${\tilde{\mathcal{O}}}\left(\frac{(1-\gamma \tau)SA}{(1-\gamma)^4\tau\epsilon^2}\right)$. 
  For a constant risk level $0<\tau\leq 1$, our upper and lower bounds match with each other, demonstrating the tightness and optimality of our analyses.
  We also investigate a limiting case with a small risk level $\tau$, called Worst-Path RL, where the objective is to maximize the minimum possible cumulative reward. We develop matching upper and lower bounds of $\tilde{\mathcal{O}}\left(\frac{SA}{p_{\min}}\right)$, where
  $p_{\min}$ denotes the minimum non-zero reaching probability of the transition kernel.
  
\end{abstract}

\section{Introduction}

Reinforcement learning (RL) \citep{sutton2018reinforcement} is a foundational framework for solving sequential decision-making problems, and it finds a wide range of applications in e.g., large language models \citep{ouyang2022training}, robotics \citep{kober2013reinforcement}, and finance \citep{charpentier2021reinforcement}. Recently there has been a surge of interest in understanding its fundamental sample complexity, e.g., \cite{bhandari2018finite,srikant2019,zou2019finite,agarwal2021theory,bhandari2024global}. However, the main focus was on the risk-neural setting, where the objective is to maximize the expected total rewards accumulated over time. As RL is increasingly applied to real-world sequential decision-making tasks, it often becomes essential to account for risk rather than simply optimizing for the total reward. This is particularly important when safety and worst-case avoidance considerations exist, e.g., in finance and investment, healthcare, autonomous systems, and process engineering. Nevertheless, a fundamental understanding of the sample complexity for risk-sensitive RL remains largely unexplored. In this paper, we focus on one of such problems and aim to understand the fundamental sample complexity for iterated CVaR RL with a generative model.

A widely used risk measure is called coherent risk measure, which satisfies the following four properties: (i) monotonicity;(ii) translation invariance; (iii) sub-additivity; and (iv) positive homogeneity
\citep{artzner1999coherent}.
Conditional Value at Risk (CVaR) is a popular coherent risk measure\citep{rockafellar2000optimization}. There are two types of CVaR RL objectives, Iterated (dynamic) CVaR RL \citep{hardy2004iterated} and Static CVaR RL \citep{Optimal_static_CVaR_wensun,low_rank_CVaR}. Iterated CVaR is a special case of Markov coherent risk \citep{tamar2015policy} where the coherent risk measure is CVaR. It has an iterative structure and focuses on the worst portion of the reward at each step. In Iterated CVaR RL problem, the agent aims to maximize the average of the worst portion at every step. Static CVaR RL \citep{Optimal_static_CVaR_wensun,low_rank_CVaR}) on the other hand, aims to maximize the CVaR of the total reward.

Static CVaR RL and Iterated CVaR RL are quite different. In Static CVaR RL, the optimal policy is history-dependent and not stationary. In \cite{bauerle2011markov}, it is shown that Static CVaR RL can be optimally solved by resolving to standard, risk-neutral RL in an augmented MDP. In Iterated CVaR RL, the optimal policy is Markovian and stationary if the environment is stationary.
Broadly speaking, Static CVaR RL is more concerned with the overall risk of the total reward and may permit the agent to visit catastrophic states as long as the risk of the cumulative reward remains acceptable. In contrast, Iterated CVaR RL assesses risk at each step, offering a more cautious approach by preventing the agent from entering catastrophic states at any point in the trajectory.

In this paper, we focus on Iterated CVaR RL, and we aim to theoretically understand the sample complexity with access to a generator \citep{kearns1998finite}. With access to a generative model, one can draw samples from the transition kernel of a Markov decision process (MDP) conditioned on any arbitrary state-action pair. We then take a model-based approach, where we first construct a maximum likelihood estimate of the transition kernel and find the optimal policy for the estimated MDP model. In risk-neural RL, it was shown that such an approach achieves the minimax optimal sample complexity \citep{Azar_crude,Argawal_generative_model}. In Iterated CVaR RL, the goal is to find the policy that maximizes the worst $\tau$-portion of reward-to-go at each step. Clearly, this ensures that policy avoids getting into catastrophic states, prioritizing risk-sensitive behavior. Since the objective focuses on ignoring certain "good" states and prioritizing the worst-case scenarios, we undoubtedly require more samples to learn a safe policy. This naturally raises the question: how many samples are needed to produce an optimal risk-sensitive policy?


In this paper, 
we make the connection between Iterated CVaR RL and distributionally robust RL \citep{iyengar2005robust,nilim2004robustness} using the dual form of CVaR. We then decompose the error using approaches in robust RL \citep{Robust_generative_model} and develop novel analytical techniques to bound the change in CVaR resulting from approximation error. We prove that a value iteration-based algorithm, ICVaR-VI, achieves an $\epsilon$-optimal policy with at most $\tilde{\mathcal{O}}\left(\frac{SA}{(1-\gamma)^4\tau^2\epsilon^2}\right)$ samples, where $\gamma$ is the discount factor, $\tau$ is the risk tolerance, and $S, A$ are the sizes of the state and action spaces. 
Moreover,   if $\tau \geq \gamma$, we further derive an improved sample complexity of  $\tilde{\mathcal{O}}\left( \frac{SA}{(1-\gamma)^3\epsilon^2} \right)$, which actually matches with the minimax optimal sample complexity for risk-neural RL (also see Table \ref{table:compare}). 
We then develop a minimax lower bound that for any $\tau$ and $\gamma$, there always exists an MDP, for which at least $\tilde{{\mathcal{O}}}\left(\frac{(1-\gamma \tau)SA}{(1-\gamma)^4\tau\epsilon^2}\right)$ are needed.  Comparing our upper and lower bounds, they match with each other in the order of sizes of state and action spaces $S, A$ and effective horizon $(1-\gamma)^{-1}$ when the risk level is a constant in $(0,1]$.

Finally, we investigate a limiting case, named Worst-Path RL \citep{Finite_horizon_Iterated_CVaR},  where we consider a small risk level $\tau$ smaller than the minimum non-zero reaching probability of the transition kernel $p_{\min}$. In this case, the CVaR risk measure actually tries to find the worst-case state. This problem cannot be directly solved by taking the limit $\tau \to 0$, and the previous lower and upper bounds will also go to infinity as they depend on $\tau^{-1}$. To tackle this problem, we design a new algorithm based on the reduced Bellman operator and develop matching upper and lower bounds of $\tilde{\mathcal{O}}\left(\frac{SA}{p_{\min}}\right)$, where
  $p_{\min}$ denotes the minimum non-zero reaching probability of the transition kernel.

\begin{table*}[ht!]
\centering
\begin{tabular}{c c c }
\hline
\textbf{Problem}          & \textbf{Lower bound} & \textbf{Upper bound}                                   \\ \hline
Risk-neutral \citep{Argawal_generative_model}  & $\tilde{\mathcal{O}}\left(\frac{SA}{(1-\gamma)^3 \epsilon^2}\right)$           & $\tilde{\mathcal{O}}\left(\frac{SA}{(1-\gamma)^3 \epsilon^2}\right)$ \\ 
Episodic, Iterated CVaR \citep{Finite_horizon_Iterated_CVaR}       & $\tilde{\mathcal{O}}\left(\frac{A}{\epsilon^2 \tau^{(1-\gamma)^{-1}-1}} \right)$            & $\tilde{\mathcal{O}}\left(\frac{S^2A}{(1-\gamma)^4 \tau^{(1-\gamma)^{-1}+1} \epsilon^2}\right)$ \\ 
Episodic, Iterated OCE \citep{xu2023regret}              &      $\mathcal{O}\left(\frac{SA}{(1-\gamma)^3\tau\epsilon^2} \right)$       & $\tilde{\mathcal{O}}\left(\frac{\left(\tau^{-\frac{1}{2}(1-\gamma)^{-1}} -1 - (1-\gamma)^{-1}(\tau^{-1/2}-1)\right)^2 S^2A}{(1-\sqrt{\tau})^4\epsilon^2} \right)$\\ 
\hline
Infinite horizon, Iterated CVaR (\textbf{This work})                        &  $\tilde{\mathcal{O}}\left( \frac{(1-\gamma \tau)SA}{(1-\gamma)^4\tau\epsilon^2} \right)$        & $\substack{\tilde{\mathcal{O}}\left(\frac{SA}{(1-\gamma)^4 \tau^{2} \epsilon^2}\right) \\
\tilde{\mathcal{O}}\left(\frac{SA}{(1-\gamma)^3  \epsilon^2}\right) \text{ if } \tau\geq\gamma}$
\\
\hline
\end{tabular}
\caption{Comparison of lower and upper bound for RL with Iterative Risk Measures. Some of the results are presented in terms of regret bound, and we converted them to sample complexity for the ease of comparison.
}\label{table:compare}
\end{table*}



\section{Related Work}

\textbf{Static CVaR RL:}
There is a long line of works focusing on static CVaR RL, which refers to the CVaR (i.e. the worst $\tau$ portion) of the accumulated total reward. 
\citet{Bastani2022regretbound} proved the first regret bound, and \citet{Optimal_static_CVaR_wensun} improved the results to be minimax optimal. \citet{low_rank_CVaR} introduced function approximation to the MDP structure and studied static CVaR in low-rank MDPs. Additionally, \citet{Risk-sensitive_RFE} developed the sample complexity of reward-free exploration in static CVaR. However, static CVaR is intrinsically different from the iterated CVaR RL studied in this paper. Iterated CVaR RL concerns the worst $\tau$-percent of the reward-to-go at each step. Intuitively, static CVaR takes
more cumulative reward into account and prefers actions that have better performance in general, while iterated CVaR prevents the agent from getting into catastrophic states \citep{Finite_horizon_Iterated_CVaR}.
Therefore, the algorithm designs and analysis techniques introduced above can not be applied to our problem.

\textbf{Iterated CVaR RL:} 
\cite{hardy2004iterated} first introduced the Iterated CVaR and showed that it is a coherent risk measure. \citet{osogami2012iterated}, \citet{chu2014markov} and \citet{bauerle2022markov} studied the iterated coherent risk measures
(iterated CVaR included) and proved the existence of a Markovian deterministic optimal policy. \citet{bauerle2022markov} also established a connection between iterated coherent risk measures and distributional robust MDPs. \citet{chen2023_iterated_function_approximation} studied iterated CVaR with function approximation and human feedback.

For the more general iterated coherent risk measures (Markov coherent risk), \citet{tamar2015policy} derived the policy gradient algorithm for both static and dynamic (iterated) coherent risk measures. \citet{huang2021convergence} then proved that gradient dominant doesn't hold for iterated coherent risk measures and that stationary point is not guaranteed to be globally optimal.

To the best of our knowledge, \citet{Finite_horizon_Iterated_CVaR} is the most related work to ours, where the Iterated CVaR objective in the episodic setting was studied. Their sample complexity results are listed in \Cref{table:compare}. In this paper, we obtain tighter upper and lower bounds, which are polynomial in the effective horizon $(1-\gamma)^{-1}$, and have a better dependence on the number of states $S$. Moreover, our bounds are minimax optimal for almost all choices of risk level.
\cite{xu2023regret} studied the recursive optimized certainty equivalent (OCE) problem in an episodic setting, where OCE is a more generalized risk measure, including CVaR. For Iterated CVaR RL, the upper and lower bounds in their paper are listed in Table \ref{table:compare}. Compared to their results, our upper and lower bounds are much tighter (minimax optimal for almost all choices of risk level) and have a clear and easy-to-understand dependence on relevant factors.

\textbf{Sample Complexity of Distributionally Robust RL with a Generative Model:}
The Iterative CVar RL problem can be equivalently written as a distributionally robust RL problem with a certain uncertainty set. The fundamental sample complexity for distributionally robust RL with a generative model has been studied in the literature for uncertainty sets defined by e.g., total variation \citep{Robust_generative_model,panaganti2022sample,yang2022toward}, $\chi^2$-divergence \citep{Robust_generative_model,panaganti2022sample,yang2022toward}, Kullback-Leibler (KL) divergence \citep{2022laixishi_KL_robust}. 

\section{Preliminaries and Problem Formulation}
\label{problem formulation}

\textbf{Notations.} We denote by $\Delta(\mathcal{S})$ the probability simplex over a set $\mathcal{S}$. In this work, we use the standard $\mathcal{O}(\cdot)$ notation to hide universal constant factors and use $\tilde{\mathcal{O}}(\cdot)$ to hide logarithmic factors. 

\textbf{Conditional Value-at-Risk (CVaR).} We begin by introducing two risk measures: value-at-risk (VaR) and conditional value-at-risk (CVaR). Let $Z$ be a random variable with cumulative distribution function $F_Z(z) = P(Z \leq z)$. The Value-at-risk at risk level $\tau\in(0,1]$ is defined as
\begin{align}
    \text{VaR}_\tau(Z) = \inf \left\{ z: F_Z(z) \geq \tau\right\}.
\end{align}
The conditional value-at-risk at risk level $\tau\in(0,1]$ is defined as
\begin{align}\label{definition of CVaR}
    \text{CVaR}_{\tau} ( Z ) = \sup_{ s \in \mathbb{R} }\left\{ s - \frac{ 1 }{ \tau } \mathbb{E} \left[ \left( s - Z\right)_+ \right]\right\}, 
\end{align}
where $(x)_+ = \max \{x,0\}$ for some $x\in\mathbb R$. If $F_Z(z)$ is continuous at $\text{VaR}_\tau(Z)$, then CVaR can also be equivalently written as \citep{Lectures_on_stochastic_programming}: 
\begin{align}
    \text{CVaR}_{\tau} ( Z ) = \mathbb{E} \left[Z|Z \leq \text{VaR}_\tau(Z)\right].
\end{align}
From the above equation, CVaR can be viewed as the average of the worst $\tau$-fraction of $Z$. When $\tau = 1$, $\text{CVaR}_{\tau} ( Z ) = \mathbb{E}[Z]$, and when $\tau \to 0$, $\text{CVaR}_{\tau} ( Z ) \rightarrow\essinf (Z)$. 

When we need to specify the distribution of random variable $Z$, we write $\text{CVaR}(Z)$ as
\begin{equation}\label{definition of CVaR with distribution}
    \text{CVaR}^{\tau}_{Z\sim P} ( Z ) = \sup_{ s \in \mathbb{R} }\left\{ s - \frac{ 1 }{ \tau } \mathbb{E}_{Z\sim P} \left[ \left( s - Z\right)_+ \right]\right\},
\end{equation}
where $P$ is the distribution of $Z$.

CVaR  can be equivalently written in the dual formulation \citep{Lectures_on_stochastic_programming} using the risk envelope
\begin{equation}
    \mathcal{U}_{\text{CVaR}_\tau}(P) = \left\{\xi: \xi \in \left[0,\frac{1}{\tau}\right], \ \mathbb{E}_{P}\left[\xi\right] = 1\right\},
\end{equation}
and
\begin{equation}\label{dual form risk measure}
    \text{CVaR}_{\tau}(Z) = \inf_{\xi \in \mathcal{U}_{\text{CVaR}_{\tau}}(P)} \mathbb{E}_{P} [\xi Z].
\end{equation}

\textbf{Standard Markov Decision Process.}
We denote an infinite horizon discounted Markov decision process (MDP) by a tuple $\mathcal{M} =(\mathcal{S},\mathcal{A},\gamma,P,r)$, where $\mathcal{S}$ and $\mathcal{A}$ are finite state and action spaces, $\gamma \in [0,1)$ is the discount factor, $P: \mathcal{S}\times\mathcal{A} \to \Delta(\mathcal{S})$ denotes the transition kernel that maps a state-action pair to a probability distribution over $\mathcal{S}$, and $r: \mathcal{S} \times \mathcal{A} \to [0,1]$ is the deterministic reward function. Let $S$ and $A$ denote the sizes of the state and action spaces, respectively. A stationary policy is defined by $\pi: \mathcal{S} \to \Delta(\mathcal{A})$. The value function of a policy $\pi$ for state $s$ is defined by
\begin{align}\label{risk neutral objective}
      \mathbb{E}_{\pi}\left[\sum_{t = 0}^\infty \gamma^t r(s_t,a_t) |s_0 = s\right], \forall s \in \mathcal{S},
\end{align}
where $s_t$ and $a_t$ denote the state and action at step $t$.

\textbf{Iterated CVaR Objective.}
    The risk-neutral objective defined in (\ref{risk neutral objective}) fails to consider the risks arising from the stochastic nature of state transitions and the agent’s policy decisions. Iterated coherent risk measures \citep{chu2014markov} have been introduced to model and evaluate these types of risks. For notational convenience, let
    \begin{align}
        r(s,\pi) &\coloneqq \sum_{a\in \mathcal{A}} \pi(a|s)r(s,a) , \\
        \rho_{s,\pi}(Z(s')) &\coloneqq \sum_{a\in \mathcal{A}} \pi(a|s) \rho_{s,a} (Z(s')) ,
    \end{align}
    where $\rho_{s,a}$ is a one-step coherent risk measure indexed by $(s,a) \in \mathcal{S} \times \mathcal{A}$ and the distribution of $s'$ follows the transition probability of $P(\cdot|s,a)$.
    The objective of the risk-sensitive MDP is defined as follows:
    \begin{align}\label{Markov risk measures}
        &\max_{\pi}V^{\pi}(s_0), 
    \end{align}
    where 
    $\ V^\pi(s_0) =  r(s_0,\pi) + \gamma \rho_{s_0,\pi}(r(s_1,\pi) + \gamma \rho_{s_1,\pi} (r(s_2,\pi) + ... ))$. 
    The trajectory $\{s_0,s_1,s_2,...\}$ is generated from the MDP $\mathcal{M}$ and policy $\pi$. The objective $V^{\pi}$ is defined in a nested pattern rather than through a single static measure of the total discounted reward.

    In this paper, we focus on a specific risk measure, the Conditional Value-at-Risk (CVaR), and refer to the objective in (\ref{Markov risk measures}) as the iterated CVaR objective. For notational convenience, we denote
    \begin{flalign}
        \text{CVaR}^{\tau}_{s,\pi}(Z(s')) &\coloneqq \sum_{a\in \mathcal{A}} \pi(a|s) \text{CVaR}^{\tau}_{s'\sim P(\cdot|s,a)} (Z(s')) .
    \end{flalign}
    The value function and Q-function are defined as follows:
    \begin{flalign}\label{Iterated CVaR objective}
        V^\pi(s_0)&=r(s_0,\pi) + \gamma \text{CVaR}^{\tau}_{s_0,\pi}\biggl(r(s_1,a_1) \nn \\
        &\qquad + \gamma \text{CVaR}^{\tau}_{s_1,\pi} (r(s_2,\pi) + ... )\biggl) , \\
        Q^\pi(s_0,a_0) &= r(s_0,a_0) + \gamma \text{CVaR}^{\tau}_{s_1\sim P(\cdot|s_0,a_0)}(V^{\pi}(s_1)).
    \end{flalign}
    The choice of the risk-sensitive objective function in (\ref{Iterated CVaR objective}) guarantees the existence of an optimal policy and the optimal policy is Markovian \citep{chu2014markov}. In contrast, the static CVaR objective, which only applies the risk measure to the total discounted reward once, does not have this property.

\textbf{Optimal Risk-sensitive Policy and Bellman Operator.}
     As shown in \citet{chu2014markov}, there exists a deterministic stationary optimal policy $\pi^*$  that maximizes the risk-sensitive value function simultaneously for all states:
     \begin{equation}
         \forall s \in \mathcal{S}:  V^*(s) \coloneqq V^{\pi^*}(s) = \max_{\pi} V^{\pi}(s)
    \end{equation}
    \begin{equation}
         \forall (s,a) \in \mathcal{S}\times \mathcal{A}: Q^*(s,a) \coloneqq Q^{\pi^*}(s,a) = \max_{\pi} Q^{\pi}(s,a).
    \end{equation}
The corresponding Bellman (optimality) equations are as follows:
\begin{subequations}
     \begin{align}
         Q^{\pi}(s,a) &= r(s,a) + \gamma \text{CVaR}_{\tau}\left(V^{\pi}(s')\right), \\
        Q^{*}(s,a) &= r(s,a) + \gamma \text{CVaR}_{\tau}\left(V^*(s')\right),
     \end{align}
    \end{subequations}
where $s' \sim P(\cdot|s,a)$. The Bellman operator is denoted by $\mathcal{T}^\tau :\mathbb{R}^{SA}\to\mathbb{R}^{SA}$ and defined as follows:
\begin{align}
    \forall(s,a)\in \mathcal{S}\times \mathcal{A}:&
    \mathcal{T}^\tau(Q)(s,a) \coloneqq r(s,a)+ \gamma \text{CVaR}_\tau V(s), \nn\\
    &\text{with} \quad V(s) \coloneqq \max_{a}Q(s,a).
\end{align}
Since $Q^*(s,a)$ is the unique fixed point of $\mathcal{T}^\tau$, we can recover the optimal policy using a value iteration algorithm (see Algorithm \ref{ICVaR-VI algorithm}). This converges rapidly due to the $\gamma$-contraction property of the $\mathcal{T}^\tau$ operator w.r.t. the $l_\infty$ norm (Lemma \ref{gamma contraction of bellman operator}).

\textbf{Connection to Distributional Robust RL.} Applying the dual form of CVaR, the Bellman equation can be re-written as
\begin{equation}\label{dual form Bellman equation}
    Q^{\pi}(s,a) = r(s,a) + \gamma \inf_{\xi \in \mathcal{U}_{\text{CVaR}_\tau}(P)} \mathbb{E}_{P}[\xi V^{\pi}].
\end{equation}
This has the same form as the Bellman equation for distributional robust RL (\citealp{iyengar2005robust}; \citealp{nilim2004robustness}), and since $\mathbb{E}_P[\xi] = 1$: 
$P\xi \in \Delta(\mathcal{S})$ is indeed a transition kernel. The uncertainty set of the transition kernel can be defined as follows:
\begin{equation}\label{CVaR uncertainty set}
    \mathcal{U}^{\tau}(P_{s,a}) = \left\{ \varient{P}_{s,a} \in \Delta(\mathcal{S}), \ 0\leq \frac{\varient{P}_{s,a}(s')}{P_{s,a}(s')}\leq \frac{1}{\tau} \right\},
\end{equation}
and the uncertainty set satisfied the $(s,a)$-rectangularity \citep{iyengar2005robust}.
We can define the robust Bellman operator for CVaR as:
\begin{equation}
\begin{split}
    &\forall(s,a)\in \mathcal{S}\times \mathcal{A}: \\
    &\mathcal{T}^\tau(Q)(s,a) \coloneqq r(s,a)+ \gamma \inf_{\varient{P} \in \mathcal{U}^\tau(P_{s,a}) }\varient{P}V.
\end{split}
\end{equation}

\textbf{Generative Model.} Assume we have access to a generative model or a simulator, which can provide samples $s'\sim P(\cdot|s,a)$, for any $(s,a)$. Suppose we call our generative model $N$ times for each state-action pair. Let $\widehat{P}$ be the empirical model defined as follows
\begin{equation*}
    \forall(s,a) \in \mathcal{S} \times \mathcal{A}, \  s_{i,s,a} \overset{\text{i.i.d}}{\sim} P(\cdot|s,a),  i = 1,2,\cdots,N. 
\end{equation*}
The total sample size is then $NSA$.

\textbf{Goal.} Given the collected samples, the goal is to learn the risk-sensitive optimal policy under risk level $\tau$ using as few samples as possible. Specifically, given a target accuracy tolerance $\epsilon >0$, the goal is to find an $\epsilon$-optimal risk-sensitive policy $\widehat{\pi}$ s.t.
\begin{equation}
    \forall s \in \mathcal{S}: \qquad V^*(s) - V^{\widehat{\pi}}(s) \leq \epsilon.
\end{equation}

\section{Algorithm}
\label{others}

In this section, we present a model-based approach, which first constructs an empirical nominal transition kernel based on the collected samples and then applies a value iteration-based algorithm ICVaR-VI to compute an optimal risk-sensitive policy in the approximated MDP.

\textbf{Empirical Nominal Transition Kernel.} The empirical nominal transition kernel $\widehat{P}$ can be constructed as follows: $\forall(s,a) \in \mathcal{S} \times \mathcal{A}$
\begin{equation}\label{emprical kernel}
     \widehat{P}(s'|s,a) \coloneqq \frac{1}{N}
    \sum_{i = 1}^{N} \mathds{1}\big\{ s_{i,s,a} = s'\big\}.
\end{equation}
We define $\widehat{\mathcal{M}}$ to be the empirical MDP that is identical to the original M, except that it uses $\widehat{P}$ instead of $P$ for the transition kernel. We use $\widehat{V}^\pi$ and $\widehat{Q}^{\pi}$ to denote the value and action value functions of a policy $\pi$ in $\widehat{\mathcal{M}}$. And $\widehat{\pi}^*$, $\widehat{Q}^*$ and $\widehat{V}^*$ are the optimal policy and value functions in $\widehat{\mathcal{M}}$.

Equipped with $\widehat{P_0}$, we can define the empirical Bellman operator $\mathcal{T}^\tau$ as follows: $\forall(s,a) \in \mathcal{S} \times \mathcal{A}$, 
\begin{equation}\label{empirical risk-sensitive bellman operator}
\widehat{\mathcal{T}}^\tau(Q)(s,a) \coloneqq r(s,a) + \gamma \text{CVaR}^\tau_{s' \sim P(\cdot|s,a)}(V(s')),
\end{equation}
where $V(s) = \max_{a} Q(s,a)$.

\textbf{ICVaR-VI: Iterated CVaR Value Iteration.} To find the fixed point of $\mathcal{T}^{\tau}$, we introduce iterated CVaR value iteration (ICVAR-VI)\citep{ruszczynski2010risk}, which is shown in Algorithm \ref{ICVaR-VI algorithm}. The update rule can be written as:
\begin{align}\label{CVaR bellman update}
    \widehat{Q}_t(s,a) &= \widehat{\mathcal{T}^\tau}(\widehat{Q}_{t-1})(s,a)\nn \\ &= r(s,a) + \gamma \text{CVaR}_{\tau}(\widehat{V}_{t-1}(s')), 
\end{align}
where $\widehat{V}_{t-1} = \max_{a}\widehat{Q}_{t-1}(s,a)$ for all $s\in\mathcal{S}$. 

\begin{algorithm}[H]
\caption{ICVaR-VI}
\label{ICVaR-VI algorithm}
\begin{algorithmic}[1]
\Require Empirical nominal transition kernel $\widehat{P}$; reward function $r$; risk level $\tau$; number of iterations $T$.
\State \textbf{Initialization:} $\widehat{Q}_0(s, a) = 0$, $\widehat{V}_0(s) = 0$ for all $(s, a) \in \mathcal{S} \times \mathcal{A}$.
\For{$t = 1, 2, \dots, T$}
    \For{$s \in \mathcal{S}, a \in \mathcal{A}$}
        \State Update $\widehat{Q}_t(s, a)$ according to (\ref{CVaR bellman update}); 
    \EndFor
    \For{$s \in \mathcal{S}$}
        \State Set $\widehat{V}_t(s) = \max_{a} \widehat{Q}_t(s, a)$;
    \EndFor
\EndFor
\State \textbf{Output:} $\widehat{Q}_T$, $\widehat{V}_T$, and policy $\widehat{\pi}(s) = \arg \max_a \widehat{Q}_T(s, a)$.
\end{algorithmic}
\end{algorithm}


\section{Theoretical Results}
\label{Theoretical Results section}
In this section, we present our main theoretical results. We start with the upper bound on the sample complexity for Iterated CVaR RL.
\begin{theorem}[Sample Complexity Upper Bound]\label{Theorem upper}
For any risk level $\tau \in (0,1]$, the number of samples needed by Algorithm ICVaR-VI to return an $\epsilon$-optimal policy with probability at least $1 - \delta$ is at most 
    $
        \tilde{\mathcal{O}}\biggl(\frac{SA}{\tau^2(1-\gamma)^4\epsilon^2} \biggl).
 $
In addition, when $\tau \geq \gamma$, the sample complexity can be further improved to $\tilde{\mathcal{O}}\left( \frac{SA}{(1-\gamma)^3\epsilon^2} \right)$.
\end{theorem}

\begin{remark}
    In Theorem \ref{Theorem upper}, the dependencies on $S$ and $A$ match with the result of Risk-neutral RL \citep{Azar_crude}. Furthermore, our sample complexity matches the dependency on $\tau$ and $1-\gamma$ with the result in  \citet{Finite_horizon_Iterated_CVaR}\footnote{In the finite-horizon episodic setting, the horizon $H$ is analogous to $(1 - \gamma)^{-1}$ in the infinite-horizon setting. \citet{Finite_horizon_Iterated_CVaR} focuses on the stationary finite-horizon episodic setting, where the transition kernel at each time step is the same, and therefore we incorporate an extra $H$ into the sample complexity for a fair comparison with our results.}, and our results improve the dependency on $S$ by a factor of $S$ \citep{Finite_horizon_Iterated_CVaR}. 
\end{remark}

The sample complexity upper bound holds for any risk level $\tau \in (0,1]$. In the special case when $\tau = 1$, CVaR reduces to expectation, and Iterated CVaR RL reduces to risk-neutral RL. In this case, our sample complexity bound is $\tilde{\mathcal{O}}\big(\frac{SA}{(1-\gamma)^3\epsilon^2} \big )$, which matches the result of the state-of-the-art sample complexity bound for standard risk-neutral RL (\citealp{Argawal_generative_model}; \citealp{Azar_crude}). 
Furthermore, for an arbitrary constant risk level $0<\tau < \gamma$, the sample complexity is increased by a factor of $(1-\gamma)^{-1}$, which, later in our lower bound analysis, is proved to be necessary.

In standard risk-neutral RL, the sample complexity of $(1-\gamma)^{-3}$ can be obtained using Bernstein inequality in combination with the Bellman property of a policy's variance (Lemma 4 of \citet{Argawal_generative_model}), which could improve the sample complexity by a factor of  $(1-\gamma)^{-1}$ comparing to only using Hoeffding's inequality.
However,  using Bernstein inequality may not always improve the sample complexity for Iterated CVaR RL by a factor of $(1-\gamma)^{-1}$. In distributional robust RL (which is similar to our case since the dual form of Iterated CVaR can also be written as a distributionally robust optimization form), there is an extra term because the worst-case transition kernel is different from the nominal one, and whether Bernstein is superior to Hoeffding's inequality depends on the specific uncertainty set (\citealp{Robust_generative_model}; \citealp{2022laixishi_KL_robust};\citealp{panaganti2022sample}). In our setting, when $\tau < \gamma$, applying Bernstein's inequality to the uncertainty set for CVaR leads to the same sample complexity bound as Hoeffding's inequality. This phenomenon also appears in existing sample complexity analyses for $\chi^2$ distributionally robust RL \citep{2022laixishi_KL_robust}. Still, for total-variation distance defined robust RL, Bernstein's inequality can improve the upper bound when $\tau \geq \gamma$, the total-variation distance of our uncertainty set is upper bounded by $\frac{1-\tau}{\tau}$, and Bernstein's inequality can reduce the sample complexity by a factor of $(1-\gamma)^{-1}$.

 The analysis of Iterated CVaR RL presents the following challenge:
The objective is not the expected cumulative reward but rather the Iterated CVaR risk measure, which prevents us from decomposing the error as we do in risk-neutral cases.
To address this challenge, we will establish a connection between Iterated CVaR RL and distributionally robust RL. This connection enables us to treat the risk-sensitive objective as the expected cumulative reward under the worst-case transition kernel. We then quantify the deviation between the empirical model and the true underlying model, using Hoeffding's inequality to derive bounds for $\tau < \gamma$. When $\tau \geq \gamma$, we introduce an alternative analytical approach using Bernstein inequality to further tighten the sample complexity bounds.

Below we present a proof sketch for Theorem \ref{Theorem upper} to highlight our major technical contributions (also see Appendix \ref{appendix upper bound} for a complete proof).
\begin{proof}[\textit{Proof sketch of Theorem \ref{Theorem upper}.}]
 With the connection between Iterated CVaR and distributionally robust RL, we can decompose the error in the following way. Let $\widehat{\pi}^*$ denote the optimal risk-sensitive policy in the empirical model $\widehat{\mathcal{M}}$, and let $\widehat{\pi}$ represent the policy from Algorithm \ref{ICVaR-VI algorithm}. Additionally, $\widehat{V}$ represents the Iterated CVaR value function in the empirical model.

The sub-optimality gap between $\pi^*$ and $\widehat{\pi}$ can be decomposed as
\begin{align}\label{decomposing the error sketch}
        V^* - V^{\widehat{\pi}} 
        & \overset{\text{(i)}}{\leq} \left(V^{\pi^*} - \widehat{V}^{\pi^*}\right)  + \frac{2\gamma\epsilon_{opt}}{1-\gamma}\mathbf{1} + \left(\widehat{V}^{\widehat{\pi}} - V^{\widehat{\pi}}\right),
\end{align}
where (i) holds by the optimality of $\widehat{\pi}$ in $\widehat{\mathcal{M}}$ and $\gamma$-contraction property of risk-sensitive Bellman operator (Lemma \ref{gamma contraction of bellman operator}). $\mathbf{1}\in \mathbb{R}^{S}$ is the all 1 vector. 


To bound $||\widehat{V}^{\pi^*} - {V}^{\pi^*}||_{\infty}$ and $||\widehat{V}^{\widehat{\pi}} - {V}^{\widehat{\pi}}||_{\infty}$, we introduce a key inequality:
    \begin{align}\label{sketch CVaR estimate}
        &\left|\text{CVaR}^{\tau}_{s'\sim \widehat{P}(\cdot | s,a)}( V^{\pi}(s')) - \text{CVaR}^{\tau}_{s'\sim {P}(\cdot | s,a)}(V^\pi(s'))\right| \nn \\ & \leq  \frac{1}{\tau} \sup_{t\in\mathbb{R}} \left|\sum_{s'} \big(\widehat{P}_{s,a}(s') - P_{s,a}(s') \big)\edited{\min\left( V^\pi(s'), t\right)} \right|.
    \end{align}
Applying the concentration lemma for CVaR (Lemma \ref{Crude value bound lemma}), we get that 
\begin{equation}\label{env error sketch}
    \begin{alignedat}{1}
    ||\widehat{V}^\pi - V^\pi||_{\infty} \leq c_0\sqrt{\frac{L}{N\tau^2(1-\gamma)^4}} 
\end{alignedat}.
\end{equation}
$L$ stands for the $\log$ term of $S$, $A$, $N$ and $\frac{1}{\delta}$. $c_0$ is a large enough constant.
Finally, for a small enough $\epsilon_{opt}$, we obtain the sample complexity upper bound.

For the special case when $\tau \geq \gamma$, we first introduce the total-variation distance bound for the CVaR uncertainty set $\mathcal{U}^{\tau}$:
\begin{align}
    \forall \varient{P}_{s,a} \in \mathcal{U}^{\tau}(P_{s,a}): \|\varient{P}_{s,a} - P_{s,a}\|_1 \leq 2\frac{1-\tau}{\tau}.
\end{align}

This bound is useful when $\tau \geq \frac{1}{2}$. When the $\tau \geq \gamma$, the dependence on $(1-\gamma)^{-1}$ of the extra term induced by worst-case transition mismatch ($\mathcal{C}_2$ in \citet{Robust_generative_model}) is reduced when applying Berstein inequality:
\begin{align}
    & \left|\Var_{\varient{P}_{s,a}}(V) - \Var_{P_{s,a}}(V)\right| \leq\|\varient{P}_{s,a} - P_{s,a}\|_1 \|V\|_{\infty}^2 \nn \\
    & \qquad \leq \frac{1-\tau}{\tau}\frac{1}{(1-\gamma)^2}  \leq \frac{1}{\gamma(1-\gamma)},
\end{align}
where $\Var_{{P}_{s,a}}(V)$ is the variance of $V$ respect to distribution ${P}_{s,a}$. In this case, we can reduce the order of $(1-\gamma)^{-1}$, which leads to the following bound:
\begin{align}
    \|\widehat{V}^{\pi} - V^{\pi}\|_{\infty} \leq c_1 \sqrt{\frac{L}{N(1-\gamma)^3}},
\end{align}
where $c_1$ is some universal constant.

The analysis of the above two cases concludes the proof of the sample complexity upper bound to get an $\epsilon$-optimal policy.
\end{proof}

Additionally, in order to assess the tightness of Theorem \ref{Theorem upper}, we further develop a minimax lower bound as follows, with the proof provided in Appendix \ref{appendix lower bound}.
\begin{theorem}[Sample Complexity Lower Bound]\label{Theorem lower bound}
Fix any $\tau \in \big(0,1\big]$, $\gamma \in \left(\frac{1}{2},1\right)$, there exist an MDP s.t. for any algorithm to obtain an $\epsilon$-optimal policy, the sample complexity is at least
    $
        \tilde{\mathcal{O}}\biggl(\frac{(1-\gamma \tau)SA}{\tau(1-\gamma)^4\epsilon^2}\biggl).
    $
    In addition, when $\tau\geq\gamma$, the sample complexity of any algorithm is at least $\tilde{\mathcal{O}}\left(\frac{SA}{(1-\gamma)^3\epsilon^2}\right)$.
\end{theorem}

\begin{remark}
    Theorem \ref{Theorem lower bound} shows that when $\tau$ is small, there exists an MDP for which the sample complexity becomes unavoidably large. When the risk tolerance is low—indicating that the agent is highly sensitive to adverse states—more samples are needed for each state-action pair to gather more accurate information about the environment, allowing the agent to develop a safer policy.
\end{remark}

\textit{{Lower Bound Analysis.}} In this following, we outline the proof idea for the lower bound in Theorem \ref{Theorem lower bound}, with the full proof deferred to Appendix \ref{appendix lower bound}. Our proof is inspired by the lower bound construction in distributional robust RL \citep{Robust_generative_model}. We first construct two similar 
  MDPs with close transition kernels that are hard to distinguish. For each MDP, there is an unknown optimal action.  If an algorithm is capable of achieving an $\epsilon$-optimal policy in the Iterated CVaR RL problem, it must also be able to identify the optimal action and determine which MDP it is interacting with, with high probability.
The challenge then becomes determining how many samples are needed to distinguish between two distributions. One notable difference in our problem is that, in our case, the two transition probabilities to the rewarding state are close to $1-\tau$, rather than $1-\gamma$ as in standard and distributional robust RL. Since CVaR computes the average over the worst $\tau$-quantile of the reward-to-go, the transition probability to lower reward states must be smaller than $\tau$ for the worst-case transition kernel to differ. As a result, the final lower bound has a higher order in $1-\gamma$ compared with risk-neutral settings when $\tau$ is relatively small.

Recall that the sufficient and necessary sample complexity for learning a standard risk-neutral MDP is $\tilde{\mathcal{O}} \left(\frac{SA}{(1-\gamma)^3\epsilon^2}\right)$ (\citealp{Azar_crude}; \citealp{Argawal_generative_model}). Intuitively, the sample complexity for CVaR should include an additional $\frac{1}{\tau}$ factor compared to the risk-neutral case, since CVaR only considers the worst $\tau$-portion of outcomes and takes an average. Therefore, the number of samples needed should be $\tilde{\mathcal{O} }\left(\frac{SA}{\tau(1-\gamma)^3\epsilon^2}\right)$. This is true for static CVaR RL where \citet{Optimal_static_CVaR_wensun} provided a regret lower bound has an extra $\sqrt{\tau^{-1}}$ term comparing to risk-neutral regret lower bound ($\sqrt{\tau^{-1}}$ in regret is equivalent to $\tau^{-1}$ in PAC condition). But for Iterated CVaR RL, it is not merely a matter of averaging the worst $\tau$-protion of trajectories. \citet{Finite_horizon_Iterated_CVaR} in Section C.2 provides a detailed discussion of the differences between static and iterated CVaR in the episodic setting. 


When $\tau \geq \gamma$, our sample complexity lower bound becomes $\tilde{\mathcal{O}}\left( \frac{SA}{(1-\gamma)^3\epsilon^2}\right)$ which matches the minimax optimal sample complexity for the risk-neutral case, in general, Iterated CVaR RL is harder to learn than standard RL. However, when $\tau$ is large, then the problem becomes closer to a risk-neutral one.

\textbf{Nearly Tight Sample Complexity.}
By combining the upper bound from Theorem \ref{Theorem upper} with the minimax lower bound from Theorem \ref{Theorem lower bound}, we confirm that the sample complexity is nearly optimal:
\begin{itemize} 
\item When $\tau \in(0,1)$ is a constant independent of $\gamma$, our sample complexity upper bound $\tilde{\mathcal{O}}\big(\frac{SA}{(1-\gamma)^4\epsilon^2}\big)$ is tight and matches the minimax lower bound; 
\item When $\tau \geq \gamma$, our sample complexity upper bound  is $\tilde{\mathcal{O}}\left(\frac{SA}{(1-\gamma)^3\epsilon^2}\right)$, and it matches with the minimax lower bound;
\item When $\tau \leq 1-\gamma$, our sample complexity $\tilde{{\mathcal{O}}}\big(\frac{SA}{\tau^2(1-\gamma)^4\epsilon^2}\big)$ has a gap of $\frac{1}{\tau}\geq\frac{1}{1-\gamma}$ compared to the minimax lower bound.
This case of small risk level $\tau$ will be further discussed in the next section.
\end{itemize}
\section{Worst-Path RL}\label{section:worst_path}
In this section, we investigate the problem with a fixed MDP and consider a limiting case where the risk level is small. This problem is referred to as worst-path RL \citep{Finite_horizon_Iterated_CVaR}.

Specifically, consider an MDP, and denote by $p_{\min}$ the minimum non-zero reaching probability from any state-action pair:
    $\forall (s,a) \in \mathcal{S}\times\mathcal{A}, \ \text{and} \  \forall s'\in \text{supp}(P(\cdot|s,a)), P(s'|s,a) \geq p_{\min}$.
Consider small risk level $\tau$: $\tau\leq p_{\min}$. Here we use $\text{supp}(P)$ to denote the support of a distribution $P$.

This case is not covered by results in  \Cref{Theorem upper,Theorem lower bound}. Obviously, the sample complexity in  \Cref{Theorem upper,Theorem lower bound} depends on $\frac{1}{\tau}$ and goes to infinity as $\tau\rightarrow 0$. However, as will be shown later, in the case with $\tau\leq p_{\min}$,  a sample complexity of $\tilde{\mathcal O}(\frac{SA}{p_{\min}})$ is minimax optimal, and it does not depend on $\epsilon$ and $1-\gamma$.

Taking the minimax lower bound result in \Cref{Theorem lower bound} as an example to explain the difference. To prove the minimax lower bound in \Cref{Theorem lower bound}, for any risk level $\tau$, we construct two MDPs with transition probabilities to lower-reward states smaller than the risk level $\tau$. 
However, such hard examples do not satisfy the problem setting in this section: $\tau\leq p_{\min}.$ 

As we have mentioned earlier, as $\tau\rightarrow 0$, $\text{CVaR}_{\tau}(Z)\rightarrow\essinf Z$. In \citet{Finite_horizon_Iterated_CVaR},  regret bounds independent of the number of episodes were developed, where the bounds also depend on visitation probability to the worst state. Given access to a generative model, the problem is simpler since there is no need to explore. Then the transition to the worst states is simply the frequency of that state in the $N$ samples generated. When $\tau$ is smaller than all possible non-zero transition probability, $\text{CVaR}_\tau$ simply reduces to finding the worst state for the reward-to-go. 




\textbf{Bellman operator and Bellman equations}.
Since $\tau \leq p_{\min}$, the objective reduces to maximizing the accumulative reward along the worst-case trajectory. The Bellman equations can then be written as follows:
\begin{align}
    Q^{\pi}(s,a) &= r(s,a) + \gamma \min_{s'\in \text{supp}(P(\cdot|s,a))}(V^{\pi}(s')), \nn \\
        V^{\pi}(s) &= \max_{a\in\mathcal{A}}Q^\pi(s,a),
\end{align}
where $\min_{s'\in \text{supp}(P(\cdot|s,a))}$ considers the worst-case of all possible next states (with non-zero probability).  The algorithm for this problem is similar to Algorithm \ref{ICVaR-VI algorithm} with a slightly different Bellman operator:
\begin{align}\label{worst-path bellman operator}
    &\forall(s,a) \in \mathcal{S} \times \mathcal{A}:\nn \\
    &\quad \widehat{\mathcal{T}}^\tau(Q)(s,a) \coloneqq r(s,a) + \gamma \min_{s': n(s',s,a)>0}(V(s')),\nn \\
    &\quad \text{with} \ V(s) \coloneqq \max_{a} Q(s,a)
\end{align}
where $n(s',s,a)$ denotes the number of samples with next state $s'$ in the total N generated samples for state action pair $(s,a)$. With that in mind,  we substitute (\ref{worst-path bellman operator}) to ($\ref{empirical risk-sensitive bellman operator}$) in Algorithm \ref{ICVaR-VI algorithm}, and we have the algorithm for the worst-path RL problem in this section.

\textbf{Sample Complexity Analysis}.
Below we provide the sample complexity upper bound for the algorithm discussed above for the worst-path RL problem.
\begin{theorem}[Worst-Path RL Upper Bound]\label{theorem:worst-path upper bound}
    Consider a risk level $\tau\leq p_{\min}$. With probability at least $1-\delta$, the number of samples needed to obtain an optimal policy is at most
    \begin{equation}
        {\mathcal{O}}\left(\frac{SA}{p_{\min}}\left(1 + \log\left(\frac{SA}{\delta}\right)\right) \right).
    \end{equation}
\end{theorem}

\begin{remark}
    For any $\tau < p_{\min}$, CVaR risk measure reduces to the essential infimum, and therefore, the bound does not depend on $\tau$ anymore. More importantly, the upper bound now depends on $\frac{1}{p_{\min}}$, which is strictly smaller than $\frac{1}{\tau^2}$.
\end{remark}

The key idea in the proof is to analyze the sub-optimality gap using the occurrence of the worst-case state. If the worst-case state occurs with high probability, the Bellman operator behaves as it would under the true underlying model. Otherwise, the sub-optimality gap becomes non-vanishing. This also explains why the sample complexity does not depend on $\epsilon$ and $\frac{1}{1-\gamma}$ here. This result aligns with our intuition: if we expect a state with probability  $p_{\min}$ to occur, we need $\frac{1}{p_{\min}}$ samples for each state-action pair.

To validate the optimality of our sample complexity upper bound, we also provide a minimax lower bound.
\begin{theorem}[Worst-Path RL Lower Bound] \label{theorem:Worst-Path RL Lower Bound}
     For a given risk level $\tau$, there exists an MDP with $P_{\min}\geq\tau$ such that for any algorithm to obtain an optimal policy at a risk level $\tau\leq p_{\min}$, the sample complexity is at least 
    $
        \tilde{\mathcal O}\left( \frac{SA}{p_{\min}} \right).
    $
\end{theorem}
The minimax lower bound matches with the upper bound in Theorem \ref{theorem:worst-path upper bound}.

\section{Acknowledgements}
The work of Z. Deng and S. Zou was partially supported by NSF under Grants CCF-2438429 and ECCS-2438392. This work is also partially supported by the AFRL VFRP program.
\section{Conclusion}

In this paper, we investigate Iterated CVaR RL problem in infinite horizon discounted MDP with access to a generative model. We introduce the algorithm ICVaR-VI and provide nearly matching sample complexity upper and lower bounds. Later we study the limit case with an arbitrarily small risk level $\tau$, and provide tight upper and lower bounds.
There are several interesting directions for future work, e.g., further closing the gap between the upper and lower sample complexity bounds and extending iterated CVaR to other types of coherent risk measure or general Markov coherent risk.

\newpage
\bibliography{ref}

\begin{thebibliography}{39}
\providecommand{\natexlab}[1]{#1}
\providecommand{\url}[1]{\texttt{#1}}
\expandafter\ifx\csname urlstyle\endcsname\relax
  \providecommand{\doi}[1]{doi: #1}\else
  \providecommand{\doi}{doi: \begingroup \urlstyle{rm}\Url}\fi

\bibitem[Agarwal et~al.(2020)Agarwal, Kakade, and
  Yang]{Argawal_generative_model}
Alekh Agarwal, Sham Kakade, and Lin~F. Yang.
\newblock Model-based reinforcement learning with a generative model is minimax
  optimal.
\newblock In \emph{Proc. Annual Conference on Learning Theory (CoLT)}, volume
  125, pp.\  67--83, 2020.

\bibitem[Agarwal et~al.(2021)Agarwal, Kakade, Lee, and
  Mahajan]{agarwal2021theory}
Alekh Agarwal, Sham~M Kakade, Jason~D Lee, and Gaurav Mahajan.
\newblock On the theory of policy gradient methods: Optimality, approximation,
  and distribution shift.
\newblock \emph{Journal of Machine Learning Research}, 22\penalty0
  (98):\penalty0 1--76, 2021.

\bibitem[Artzner et~al.(1999)Artzner, Delbaen, Eber, and
  Heath]{artzner1999coherent}
Philippe Artzner, Freddy Delbaen, Jean-Marc Eber, and David Heath.
\newblock Coherent measures of risk.
\newblock \emph{Mathematical finance}, 9\penalty0 (3):\penalty0 203--228, 1999.

\bibitem[Bastani et~al.(2022)Bastani, Ma, Shen, and Xu]{Bastani2022regretbound}
Osbert Bastani, Jason~Yecheng Ma, Estelle Shen, and Wanqiao Xu.
\newblock Regret bounds for risk-sensitive reinforcement learning.
\newblock \emph{Advances in Neural Information Processing Systems},
  35:\penalty0 36259--36269, 2022.

\bibitem[B{\"a}uerle \& Glauner(2022)B{\"a}uerle and
  Glauner]{bauerle2022markov}
Nicole B{\"a}uerle and Alexander Glauner.
\newblock Markov decision processes with recursive risk measures.
\newblock \emph{European Journal of Operational Research}, 296\penalty0
  (3):\penalty0 953--966, 2022.

\bibitem[B{\"a}uerle \& Ott(2011)B{\"a}uerle and Ott]{bauerle2011markov}
Nicole B{\"a}uerle and Jonathan Ott.
\newblock Markov decision processes with {Average-Value-at-Risk} criteria.
\newblock \emph{Mathematical Methods of Operations Research}, 74:\penalty0
  361--379, 2011.

\bibitem[Bhandari \& Russo(2024)Bhandari and Russo]{bhandari2024global}
Jalaj Bhandari and Daniel Russo.
\newblock Global optimality guarantees for policy gradient methods.
\newblock \emph{Operations Research}, 2024.

\bibitem[Bhandari et~al.(2018)Bhandari, Russo, and Singal]{bhandari2018finite}
Jalaj Bhandari, Daniel Russo, and Raghav Singal.
\newblock A finite time analysis of temporal difference learning with linear
  function approximation.
\newblock In \emph{Proc. Annual Conference on Learning Theory (CoLT)}, pp.\
  1691--1692. PMLR, 2018.

\bibitem[Chagny(2016)]{chagny2016introduction}
Ga{\"e}lle Chagny.
\newblock An introduction to nonparametric adaptive estimation.
\newblock \emph{The Graduate Journal of Mathematics}, 2016\penalty0
  (2):\penalty0 105--120, 2016.

\bibitem[Charpentier et~al.(2021)Charpentier, Elie, and
  Remlinger]{charpentier2021reinforcement}
Arthur Charpentier, Romuald Elie, and Carl Remlinger.
\newblock Reinforcement learning in economics and finance.
\newblock \emph{Computational Economics}, pp.\  1--38, 2021.

\bibitem[Chen et~al.(2023)Chen, Du, Hu, Wang, Wu, and
  Huang]{chen2023_iterated_function_approximation}
Yu~Chen, Yihan Du, Pihe Hu, Siwei Wang, Desheng Wu, and Longbo Huang.
\newblock Provably efficient iterated {CVaR} reinforcement learning with
  function approximation.
\newblock \emph{arXiv preprint arXiv:2307.02842}, 2023.

\bibitem[Chu \& Zhang(2014)Chu and Zhang]{chu2014markov}
Shanyun Chu and Yi~Zhang.
\newblock Markov decision processes with iterated coherent risk measures.
\newblock \emph{International Journal of Control}, 87\penalty0 (11):\penalty0
  2286--2293, 2014.

\bibitem[Dann et~al.(2017)Dann, Lattimore, and Brunskill]{dann2017unifying}
Christoph Dann, Tor Lattimore, and Emma Brunskill.
\newblock Unifying pac and regret: Uniform pac bounds for episodic
  reinforcement learning.
\newblock In \emph{Proc. Advances in Neural Information Processing Systems
  (NeurIPS)}, volume~30, 2017.

\bibitem[Du et~al.(2022)Du, Wang, and Huang]{Finite_horizon_Iterated_CVaR}
Yihan Du, Siwei Wang, and Longbo Huang.
\newblock Provably efficient risk-sensitive reinforcement learning: {Iterated}
  {CVaR} and worst path.
\newblock \emph{arXiv preprint arXiv:2206.02678}, 2022.

\bibitem[Gheshlaghi~Azar et~al.(2013)Gheshlaghi~Azar, Munos, and
  Kappen]{Azar_crude}
Mohammad Gheshlaghi~Azar, R{\'e}mi Munos, and Hilbert~J Kappen.
\newblock Minimax {PAC} bounds on the sample complexity of reinforcement
  learning with a generative model.
\newblock \emph{Machine Learning}, 91:\penalty0 325--349, 2013.

\bibitem[Hardy \& Wirch(2004)Hardy and Wirch]{hardy2004iterated}
Mary~R Hardy and Julia~L Wirch.
\newblock The iterated {CTE}: a dynamic risk measure.
\newblock \emph{North American Actuarial Journal}, 8\penalty0 (4):\penalty0
  62--75, 2004.

\bibitem[Huang et~al.(2021)Huang, Leqi, Lipton, and
  Azizzadenesheli]{huang2021convergence}
Audrey Huang, Liu Leqi, Zachary~C Lipton, and Kamyar Azizzadenesheli.
\newblock On the convergence and optimality of policy gradient for markov
  coherent risk.
\newblock \emph{arXiv preprint arXiv:2103.02827}, 2021.

\bibitem[Iyengar(2005)]{iyengar2005robust}
Garud~N Iyengar.
\newblock Robust dynamic programming.
\newblock \emph{Mathematics of Operations Research}, 30\penalty0 (2):\penalty0
  257--280, 2005.

\bibitem[Kearns \& Singh(1998)Kearns and Singh]{kearns1998finite}
Michael Kearns and Satinder Singh.
\newblock {Finite-Sample Convergence Rates for Q-Learning and Indirect
  Algorithms}.
\newblock In \emph{Proc. Advances in Neural Information Processing Systems
  (NeurIPS)}, volume~11, 1998.

\bibitem[Kober et~al.(2013)Kober, Bagnell, and Peters]{kober2013reinforcement}
Jens Kober, J~Andrew Bagnell, and Jan Peters.
\newblock Reinforcement learning in robotics: A survey.
\newblock \emph{The International Journal of Robotics Research}, 32\penalty0
  (11):\penalty0 1238--1274, 2013.

\bibitem[Ni et~al.(2024)Ni, Liu, and Lai]{Risk-sensitive_RFE}
Xinyi Ni, Guanlin Liu, and Lifeng Lai.
\newblock Risk-sensitive reward-free reinforcement learning with {CV}a{R}.
\newblock In \emph{Proc. International Conference on Machine Learning (ICML)},
  volume 235, pp.\  37999--38017, 2024.

\bibitem[Nilim \& El~Ghaoui(2004)Nilim and El~Ghaoui]{nilim2004robustness}
Arnab Nilim and Laurent El~Ghaoui.
\newblock Robustness in {{M}arkov} decision problems with uncertain transition
  matrices.
\newblock In \emph{Proc. Advances in Neural Information Processing Systems
  (NIPS)}, pp.\  839--846, 2004.

\bibitem[Osogami(2012)]{osogami2012iterated}
Takayuki Osogami.
\newblock Iterated risk measures for risk-sensitive markov decision processes
  with discounted cost.
\newblock \emph{arXiv preprint arXiv:1202.3755}, 2012.

\bibitem[Ouyang et~al.(2022)]{ouyang2022training}
Long Ouyang et~al.
\newblock Training language models to follow instructions with human feedback.
\newblock In \emph{Proc. Advances in Neural Information Processing Systems
  (NeurIPS)}, volume~35, pp.\  27730--27744, 2022.

\bibitem[Panaganti \& Kalathil(2022)Panaganti and
  Kalathil]{panaganti2022sample}
Kishan Panaganti and Dileep Kalathil.
\newblock Sample complexity of robust reinforcement learning with a generative
  model.
\newblock In \emph{Proc. International Conference on Artifical Intelligence and
  Statistics (AISTATS)}, pp.\  9582--9602, 2022.

\bibitem[Rockafellar et~al.(2000)Rockafellar, Uryasev,
  et~al.]{rockafellar2000optimization}
R~Tyrrell Rockafellar, Stanislav Uryasev, et~al.
\newblock Optimization of conditional value-at-risk.
\newblock \emph{Journal of risk}, 2:\penalty0 21--42, 2000.

\bibitem[Ruszczy{\'n}ski(2010)]{ruszczynski2010risk}
Andrzej Ruszczy{\'n}ski.
\newblock Risk-averse dynamic programming for {Markov} decision processes.
\newblock \emph{Mathematical Programming}, 125:\penalty0 235--261, 2010.

\bibitem[Shapiro et~al.(2021)Shapiro, Dentcheva, and
  Ruszczynski]{Lectures_on_stochastic_programming}
Alexander Shapiro, Darinka Dentcheva, and Andrzej Ruszczynski.
\newblock \emph{Lectures on stochastic programming: modeling and theory}.
\newblock SIAM, 2021.

\bibitem[Shi \& Chi(2024)Shi and Chi]{2022laixishi_KL_robust}
Laixi Shi and Yuejie Chi.
\newblock Distributionally robust model-based offline reinforcement learning
  with near-optimal sample complexity.
\newblock \emph{Journal of Machine Learning Research}, 25\penalty0
  (200):\penalty0 1--91, 2024.

\bibitem[Shi et~al.(2023)Shi, Li, Wei, Chen, Geist, and
  Chi]{Robust_generative_model}
Laixi Shi, Gen Li, Yuting Wei, Yuxin Chen, Matthieu Geist, and Yuejie Chi.
\newblock The curious price of distributional robustness in reinforcement
  learning with a generative model.
\newblock In \emph{Proc. Advances in Neural Information Processing Systems
  (NeurIPS)}, 2023.

\bibitem[Srikant \& Ying(2019)Srikant and Ying]{srikant2019}
R.~Srikant and Lei Ying.
\newblock Finite-time error bounds for linear stochastic approximation and {TD}
  learning.
\newblock In \emph{Proc. Annual Conference on Learning Theory (CoLT)}, pp.\
  2803--2830, 2019.

\bibitem[Sutton \& Barto(2018)Sutton and Barto]{sutton2018reinforcement}
Richard~S. Sutton and Andrew~G. Barto.
\newblock \emph{Reinforcement Learning: An Introduction}.
\newblock The MIT Press, Cambridge, Massachusetts, 2018.

\bibitem[Tamar et~al.(2015)Tamar, Chow, Ghavamzadeh, and
  Mannor]{tamar2015policy}
Aviv Tamar, Yinlam Chow, Mohammad Ghavamzadeh, and Shie Mannor.
\newblock Policy gradient for coherent risk measures.
\newblock \emph{Advances in neural information processing systems}, 28, 2015.

\bibitem[Vershynin(2018)]{vershynin2018high}
Roman Vershynin.
\newblock \emph{{High-Dimensional Probability: An Introduction with
  Applications in Data Science }}, volume~47.
\newblock Cambridge University Press, 2018.

\bibitem[Wang et~al.(2023)Wang, Kallus, and Sun]{Optimal_static_CVaR_wensun}
Kaiwen Wang, Nathan Kallus, and Wen Sun.
\newblock Near-minimax-optimal risk-sensitive reinforcement learning with
  {CVaR}.
\newblock In \emph{Proc. International Conference on Machine Learning (ICML)},
  pp.\  35864--35907, 2023.

\bibitem[Xu et~al.(2023)Xu, Gao, and He]{xu2023regret}
Wenhao Xu, Xuefeng Gao, and Xuedong He.
\newblock Regret bounds for {Markov} decision processes with recursive
  {Optimized} {Certainty} {Equivalents}.
\newblock In \emph{Proc. International Conference on Machine Learning (ICML)},
  pp.\  38400--38427, 2023.

\bibitem[Yang et~al.(2022)Yang, Zhang, and Zhang]{yang2022toward}
Wenhao Yang, Liangyu Zhang, and Zhihua Zhang.
\newblock Toward theoretical understandings of robust {Markov} decision
  processes: Sample complexity and asymptotics.
\newblock \emph{The Annals of Statistics}, 50\penalty0 (6):\penalty0
  3223--3248, 2022.

\bibitem[Zhao et~al.(2023)Zhao, Zhan, Hu, Leung, Farnia, Sun, and
  Lee]{low_rank_CVaR}
Yulai Zhao, Wenhao Zhan, Xiaoyan Hu, Ho-fung Leung, Farzan Farnia, Wen Sun, and
  Jason~D Lee.
\newblock Provably efficient {CVaR} {{RL}} in low-rank {MDPs}.
\newblock \emph{arXiv preprint arXiv:2311.11965}, 2023.

\bibitem[Zou et~al.(2019)Zou, Xu, and Liang]{zou2019finite}
Shaofeng Zou, Tengyu Xu, and Yingbin Liang.
\newblock Finite-sample analysis for {SARSA} with linear function
  approximation.
\newblock In \emph{Proc. Advances in Neural Information Processing Systems
  (NeurIPS)}, pp.\  8665--8675, 2019.

\end{thebibliography}

\newpage
\section*{Checklist}
 \begin{enumerate}

 \item For all models and algorithms presented, check if you include:
 \begin{enumerate}
   \item A clear description of the mathematical setting, assumptions, algorithm, and/or model. [Yes]
   \item An analysis of the properties and complexity (time, space, sample size) of any algorithm. [Yes]
   \item (Optional) Anonymized source code, with specification of all dependencies, including external libraries. [Not Applicable]
 \end{enumerate}

 \item For any theoretical claim, check if you include:
 \begin{enumerate}
   \item Statements of the full set of assumptions of all theoretical results. [Yes]
   \item Complete proofs of all theoretical results. [Yes]
   \item Clear explanations of any assumptions. [Yes]     
 \end{enumerate}

 \item For all figures and tables that present empirical results, check if you include:
 \begin{enumerate}
   \item The code, data, and instructions needed to reproduce the main experimental results (either in the supplemental material or as a URL). [Not Applicable]
   \item All the training details (e.g., data splits, hyperparameters, how they were chosen). [Not Applicable]
         \item A clear definition of the specific measure or statistics and error bars (e.g., with respect to the random seed after running experiments multiple times). [Not Applicable]
         \item A description of the computing infrastructure used. (e.g., type of GPUs, internal cluster, or cloud provider). [Not Applicable]
 \end{enumerate}

 \item If you are using existing assets (e.g., code, data, models) or curating/releasing new assets, check if you include:
 \begin{enumerate}
   \item Citations of the creator If your work uses existing assets. [Yes]
   \item The license information of the assets, if applicable. [Not Applicable]
   \item New assets either in the supplemental material or as a URL, if applicable. [Not Applicable]
   \item Information about consent from data providers/curators. [Not Applicable]
   \item Discussion of sensible content if applicable, e.g., personally identifiable information or offensive content. [Not Applicable]
 \end{enumerate}

 \item If you used crowdsourcing or conducted research with human subjects, check if you include:
 \begin{enumerate}
   \item The full text of instructions given to participants and screenshots. [Not Applicable]
   \item Descriptions of potential participant risks, with links to Institutional Review Board (IRB) approvals if applicable. [Not Applicable]
   \item The estimated hourly wage paid to participants and the total amount spent on participant compensation. [Not Applicable]
 \end{enumerate}

 \end{enumerate}

\onecolumn
\onecolumn
\appendix

\section{Notations and useful lemmas}
In the appendices, with a slight abuse of notations, we use $\myMatrix{P}\in \mathbb{R}^{SA \times S}$ to denote the transition matrix of the nominal transition kernel $P$, and let $\myMatrix{P}_{s,a}$ denote its $(s,a)$-th row. Similarly, we could define the transition matrix $\widehat{{\myMatrix{P}}}$ for the empirical nominal transition kernel $\widehat{P}$. We further define the following matrix/vector notations for the convenience of presentation. Let the state space $\mathcal S=\{0,1,2,\ldots,S-1\}$ and action space $\mathcal A=\{0,1,2,\ldots, A-1\}$. A \textit{deterministic} policy $\pi$ is a mapping from the state space to the action space, i.e., $\pi(s)$ is an action in $\mathcal{A}$. 
\begin{itemize}
    \item $\vect{r} \in \mathbb{R}^{SA}$: vector form of the reward function $r$.
    \item ${\myMatrix{\Pi}^\pi} \in \{0,1\}^{S \times SA}$: projection matrix associated with a deterministic policy $\pi$:
    \[
    {\myMatrix{\Pi}^\pi} =
    \begin{pmatrix}
    \vect{e}_{\pi(0)}^\top& \mathbf{0}^\top & \cdots & \mathbf{0}^\top \\
    \mathbf{0}^\top & \vect{e} _{\pi(1)}^\top & \cdots & \mathbf{0}^\top \\
    \vdots & \vdots & \ddots & \vdots \\
    \mathbf{0}^\top & \mathbf{0}^\top & \cdots & \vect{e} _{\pi(S-1)}^\top
    \end{pmatrix},
    \]
    where $\vect{e}_{\pi(0)}^\top, \vect{e}_{\pi(1)}^\top, \ldots, \vect{e}_{\pi(S-1)}^\top \in \mathbb{R}^A$ are standard basis vectors and $\mathbf{0} \in \mathbb{R}^S$ is all zero vector.
    \item $\vect{r}_\pi\in \mathbb{R}^S$: reward vector restricted to the actions chosen by the deterministic policy $\pi$, namely, $\vect{r}_\pi(s) = r(s, \pi(s))$ for all $s \in S$ (or simply, $\vect{r}_\pi = \myMatrix{\Pi^\pi} \vect{r}$).
    \item ${\myMatrix{P}^V} \in \mathbb{R}^{SA \times S}$, $\widehat{{\myMatrix{P}}}^V \in \mathbb{R}^{SA \times S}$: worst-case transition matrices for a vector $\vect{V} \in \mathbb{R}^S$. We denote $\vect{P}^V_{s,a}$ (resp. $\widehat{\vect{P}}^V_{s,a}$) as the $(s,a)$-th row. Specifically, 
    \[
    \vect{P}^V_{s,a} = \argmin_{\varient{\vect{P}}_{s,a} \in \mathcal{U}^\tau(\vect{P}_{s,a})} \varient{\vect{P}}_{s,a} \vect{V}, \quad \widehat{\vect{P}}^V_{s,a} = \argmin_{\varient{\vect{P}}_{s,a} \in \mathcal{U}^\tau(\widehat{\vect{P}}_{s,a})} \varient{\vect{P}}_{s,a} \vect{V}. 
    \]
    Furthermore, we make use of the following short-hand notation:
    \[
    \vect{P}^{\pi,V}_{s,a} \coloneqq \argmin_{\varient{\vect{P}}_{s,a} \in \mathcal{U}^\tau(\vect{P}_{s,a})} \varient{\vect{P}}_{s,a} \vect{V}^{\pi}, \quad \vect{P}^{\pi,\widehat{V}}_{s,a} \coloneqq \argmin_{\varient{\vect{P}}_{s,a} \in \mathcal{U}^\tau(\vect{P}_{s,a})} \varient{\vect{P}}_{s,a} \widehat{\vect{V}}^{\pi}, 
    \]
    \[
    \widehat{\vect{P}}^{\pi,V}_{s,a} \coloneqq \argmin_{\varient{\vect{P}}_{s,a} \in \mathcal{U}^\tau(\widehat{\vect{P}}_{s,a})} \varient{\vect{P}}_{s,a} \vect{V}^{\pi}, \quad \widehat{\vect{P}}^{\pi,\widehat{V}}_{s,a} \coloneqq \argmin_{\varient{\vect{P}}_{s,a} \in \mathcal{U}^\tau(\widehat{\vect{P}}_{s,a})} \varient{\vect{P}}_{s,a} \widehat{\vect{V}}^{\pi}. 
    \]
    The corresponding probability transition matrices are denoted by ${\myMatrix{P}^{\pi,V}} \in \mathbb{R}^{SA \times S}$, ${\myMatrix{P}^{\pi,\widehat{V}}} \in \mathbb{R}^{SA \times S}$, $\widehat{\myMatrix{P}}^{\pi,V} \in \mathbb{R}^{SA \times S}$ and $\widehat{\myMatrix{P}}^{\pi,\widehat{V}} \in \mathbb{R}^{SA \times S}$, respectively.
    \item ${\myMatrix{P}^\pi} \in \mathbb{R}^{S \times S}$, $\widehat{\myMatrix{P}}^{\pi} \in \mathbb{R}^{S \times S}$, $\underline{\myMatrix{P}}^{\pi,V} \in \mathbb{R}^{S \times S}$, $\underline{\myMatrix{P}}^{\pi,\widehat{V}} \in \mathbb{R}^{S \times S}$, $\underline{\widehat{\myMatrix{P}}}^{\pi,V} \in \mathbb{R}^{S \times S}$ and $\underline{\widehat{\myMatrix{P}}}^{\pi,\widehat{V}} \in \mathbb{R}^{S \times S}$: probability transition matrices w.r.t. policy $\pi$ over the states:
    \begin{equation}
        \begin{split}
           &{\myMatrix{P}^\pi} \coloneqq \myMatrix{\Pi^\pi} \myMatrix{P}, \quad \widehat{\myMatrix{P}}^{\pi} \coloneqq \myMatrix{\Pi^\pi} \widehat{\myMatrix{P}}, \quad \underline{\myMatrix{P}}^{\pi,V} \coloneqq \myMatrix{\Pi^\pi} \myMatrix{P}^{\pi,V}, \quad \underline{\myMatrix{P}}^{\pi,{\widehat{V}}} \coloneqq \myMatrix{\Pi^\pi} \myMatrix{P}^{\pi,{\widehat{V}}},  \\ 
           & \underline{\widehat{\myMatrix{P}}}^{\pi,V} \coloneqq \myMatrix{\Pi^\pi} \widehat{\myMatrix{P}}^{\pi,V}, \quad \underline{\widehat{\myMatrix{P}}}^{\pi,{\widehat{V}}} \coloneqq \myMatrix{\Pi^\pi} \widehat{\myMatrix{P}}^{\pi,{\widehat{V}}}. 
        \end{split}
    \end{equation}
    \item $\Var_{\myMatrix{\varient{P}}}(\vect{V}) \in \mathbb{R}^{SA}$: for any transition kernel $\myMatrix{\varient{P}}\in \mathbb{R}^{SA\times S}$ and any vector $\vect{V} \in \mathbb{R}^S$, the $(s,a)$-th row of $\Var_{\myMatrix{P}}(\vect{V})$ is 
    \begin{align}
        \Var_{\myMatrix{P}}(s,a) = \Var_{P_{s,a}}(\vect{V}),
    \end{align}
    where $\Var_{\vect{P}_{s,a}}(\vect{V}) \coloneqq \vect{P}_{s,a}(\vect{V}^2) - (\vect{P}_{s,a}\vect{V})^2$ .
\end{itemize}
\begin{lemma}\label{gamma contraction of bellman operator}
(\citealp{ruszczynski2010risk}, Lemma 2). For any $\gamma \in [0,1)$, the robust Bellman operator $\mathcal{T}^\tau(\cdot)$ is a $\gamma$-contraction w.r.t. $\|\cdot\|_{\infty}$. Namely, for any $Q_1,Q_2\in \mathbb{R}^{SA}$ s.t. $Q_1(s,a),Q_2(s,a)\in [0,\frac{1}{1-\gamma}]$ for all $(s,a) \in  \mathcal{S} \times \mathcal{A}$, one has
\begin{align}
    \|\mathcal{T}^\tau(\vect{Q_1}) - \mathcal{T}^\tau(\vect{Q_2})\|_{\infty} \leq \gamma\|\vect{Q_1}-\vect{Q_2}\|_{\infty}.
\end{align}
    
\end{lemma}
\begin{lemma}[\citealp{Robust_generative_model}, Lemma 5]\label{Value iteration convergence lemma}
    Let $\widehat{\vect{Q}}_{0} = \mathbf{0}$. The iterates $\{\widehat{\vect{Q}}_t\}$, $\{\widehat{\vect{V}}_t\}$ of $\text{ICVaR-VI}$ (Algorithm \ref{ICVaR-VI algorithm}) obey
    \begin{align}
        \forall t \geq 0: \quad \|\widehat{\vect{Q}}_t - \widehat{\vect{Q}}^*\|_{\infty} \leq \frac{\gamma^t}{1-\gamma} \quad 
        \|\widehat{\vect{Q}}_t - \widehat{\vect{Q}}^*\|_{\infty} \leq \frac{\gamma^t}{1-\gamma}.
    \end{align}
    Furthermore, the output policy $\widehat{\pi}$ obeys
    \begin{align}\label{emprical error}
        \|\widehat{\vect{V}}^* - \widehat{\vect{V}}^{\widehat{\pi}}\|_{\infty} \leq \frac{2\gamma \epsilon_{opt}}{1-\gamma}, \quad \text{where} \quad \epsilon_{opt} \coloneqq \|\widehat{\vect{V}}^* - \widehat{\vect{V}}_{T-1}\|_{\infty}.
    \end{align}
\end{lemma}

\section{Proof of \Cref{Theorem upper}: Sample Complexity Upper Bound}\label{appendix upper bound}

\textbf{Step 1: Decomposing the error.} The optimality gap can be decomposed as
\begin{align}\label{decomposing the error}
    \vect{V}^* - \vect{V}^{\widehat{\pi}} & = \left(\vect{V}^{\pi^*} - \widehat{\vect{V}}^{\pi^*}\right) + \left(\widehat{\vect{V}}^{\pi^*} - \widehat{\vect{V}}^{\widehat{\pi}^*}\right) + \left(\widehat{\vect{V}}^{\widehat{\pi}^*} - \widehat{\vect{V}}^{\widehat{\pi}}\right) + \left(\widehat{\vect{V}}^{\widehat{\pi}} - \vect{V}^{\widehat{\pi}}\right) \nn \\
    & \overset{(\text{i})}{\leq} \left(\vect{V}^{\pi^*} - \widehat{\vect{V}}^{\pi^*}\right)  + \left(\widehat{\vect{V}}^{\widehat{\pi}^*} - \widehat{\vect{V}}^{\widehat{\pi}}\right) + \left(\widehat{\vect{V}}^{\widehat{\pi}} - \vect{V}^{\widehat{\pi}}\right) \nn \\
    & \overset{\text{(ii)}}{\leq} \left(\vect{V}^{\pi^*} - \widehat{\vect{V}}^{\pi^*}\right)  + \frac{2\gamma\epsilon_{opt}}{1-\gamma}\mathbf{1} + \left(\widehat{\vect{V}}^{\widehat{\pi}} - \vect{V}^{\widehat{\pi}}\right),
\end{align}
where (i) holds by the fact that $\widehat{\pi}^*$ is the optimal policy under transition kernel $\widehat{\myMatrix{P}}$;
and (ii) follows from \Cref{Value iteration convergence lemma}.

The first and third terms in the sub-optimality gap in (\ref{decomposing the error}) can be bounded in the same way as follows
\begin{align}
\label{hat V - V lower}
    \widehat{\vect{V}}^\pi - \vect{V}^\pi & = \vect{r}_\pi + \gamma \widehat{\underline{\myMatrix{P}}}^{\pi,\widehat{V}}\widehat{\vect{V}}^\pi - \vect{r}_\pi - \gamma \underline{\myMatrix{P}}^{\pi,V}\vect{V}^\pi \nn \\
    & =\gamma \left(\widehat{\underline{\myMatrix{P}}}^{\pi,\widehat{V}}\widehat{\vect{V}}^\pi - {\underline{\myMatrix{P}}}^{\pi,\widehat{V}}\widehat{\vect{V}}^\pi \right) + \gamma \left({\underline{\myMatrix{P}}}^{\pi,\widehat{V}}\widehat{\vect{V}}^\pi  - {\underline{\myMatrix{P}}}^{\pi,\widehat{V}}\vect{V}^\pi \right) + 
    \gamma \left({\underline{\myMatrix{P}}}^{\pi,\widehat{V}}\vect{V}^\pi  - \underline{\myMatrix{P}}^{\pi,V}\vect{V}^\pi \right) \nn \\
    & \overset{(i)}{\geq} \gamma \left(\widehat{\underline{\myMatrix{P}}}^{\pi,\widehat{V}}\widehat{\vect{V}}^\pi - {\underline{\myMatrix{P}}}^{\pi,\widehat{V}}\widehat{\vect{V}}^\pi \right) + \gamma \left({\underline{\myMatrix{P}}}^{\pi,\widehat{V}}\widehat{\vect{V}}^\pi  - {\underline{\myMatrix{P}}}^{\pi,\widehat{V}}\vect{V}^\pi \right)  \nn \\
    & =\gamma (\myMatrix{I} - \gamma\underline{{\myMatrix{P}}}^{\pi,\widehat{V}})^{-1} \left(\widehat{\underline{\myMatrix{P}}}^{\pi,\widehat{V}}\widehat{\vect{V}}^\pi - {\underline{\myMatrix{P}}}^{\pi,\widehat{V}}\widehat{\vect{V}}^\pi \right),
\end{align}

where (i) holds by the fact that \edited{$\underline{\myMatrix{{P}}}^{\pi,{V}}$} corresponds to the worst-case transition kernel in $\mathcal{U}^{\tau}(P^{\pi})$ for $V^{\pi}$. By decomposing the error in a symmetric way, we can similarly obtain that
\begin{equation}
\label{hat V - V upper}
    \widehat{\vect{V}}^\pi - \vect{V}^\pi \leq \gamma (\myMatrix{I} - \gamma\underline{\widehat{\myMatrix{P}}}^{\pi,V})^{-1} (\underline{\widehat{\myMatrix{P}}}^{\pi,V}\vect{V}^{\pi} - \underline{\myMatrix{P}}^{\pi,V}\vect{V}^\pi).
\end{equation}
Combining (\ref{hat V - V lower}) and (\ref{hat V - V upper}), we arrive at

\begin{align}
\label{same policy different environment error}
    ||\widehat{\vect{V}}^\pi - \vect{V}^\pi||_{\infty} \leq \gamma \max \big\{&||(\myMatrix{I} - \gamma\underline{{\myMatrix{P}}}^{\pi,\widehat{V}})^{-1} (\widehat{\underline{\myMatrix{P}}}^{\pi,\widehat{V}}\widehat{\vect{V}}^\pi - {\underline{\myMatrix{P}}}^{\pi,\widehat{V}}\widehat{\vect{V}}^\pi )||_{\infty} , \nn \\  &||(\myMatrix{I} - \gamma\underline{\widehat{\myMatrix{P}}}^{\pi,V})^{-1} (\underline{\widehat{\myMatrix{P}}}^{\pi,V}\vect{V}^{\pi} - \underline{\myMatrix{P}}^{\pi,V}\vect{V}^\pi)||_{\infty} \big\}.
\end{align}

\textbf{Step 2: Controlling $||\widehat{\vect{V}}^{\pi} - {\vect{V}}^{\pi}||_{\infty}$ and $||\widehat{\vect{V}}^{\widehat{\pi}} - {\vect{V}}^{\widehat{\pi}}||_{\infty}$} in (\ref{decomposing the error}). 
\begin{lemma}\label{Crude value bound lemma}
    For any $\delta \in (0,1)$, with probability at least $ 1 - \delta$, one has that
    \begin{equation}
            \left\|\underline{\widehat{\myMatrix{P}}}^{\pi,V} \vect{V}^{\pi} - \underline{\myMatrix{P}}^{\pi,V}\vect{V}^\pi\right\|_{\infty} \leq \frac{2}{\tau}\sqrt{\frac{2\log(6SAN/\delta)}{N(1-\gamma)^2}}.
    \end{equation}
\end{lemma}
\begin{proof}
Note that $V^{\pi}(s)\in \left[ 0,\frac{1}{1-\gamma}\right]$ for any $s \in \mathcal{S}$. Therefore, $\underset{{x\in\mathbb{R}}}{\sup}$ is equivalent to $\underset{{x\in [0,\frac{1}{1-\gamma}]}}{\sup}$. We first show that
\begin{align}\label{point-wise difference}
    |(\widehat{\vect{P}}_{s,a}^{\pi,V} - \vect{P}_{s,a}^{\pi,V})V^{\pi}| &=\biggl|\sup_{x\in [0,\frac{1}{1-\gamma}]}\left\{x - \frac{1}{\tau}\mathbb{E}_{s'\sim \widehat{P}_{s,a}}[(x - V^{\pi}(s'))_{+}]\right\} - \sup_{x\in [0,\frac{1}{1-\gamma}]}\left\{x - \frac{1}{\tau}\mathbb{E}_{s'\sim {P}_{s,a}}[(x - V^{\pi}(s'))_{+}]\right\}\biggl| \nn \\
    &\leq \frac{1}{\tau}\sup_{x\in [0,\frac{1}{1-\gamma}]} \left|\mathbb{E}_{s'\sim \widehat{P}_{s,a}}[(x - V^{\pi}(s'))_{+}] - \mathbb{E}_{s'\sim {P}_{s,a}}[(x - V^{\pi}(s'))_{+}] \right| \nn \\
    &= \frac{1}{\tau} \sup_{x\in[0,\frac{1}{1-\gamma}]} \left| \mathbb{E}_{s'\sim \widehat{P}_{s,a}}\big[x - \edited{\min(V^{\pi}(s'),x)}\big] - \mathbb{E}_{s'\sim {P}_{s,a}}\big[x - \edited{\min(V^{\pi}(s'),x)}\big] \right| \nn \\
    & = \frac{1}{\tau} \sup_{x\in[0,\frac{1}{1-\gamma}]} \left|\sum_{s'} \big(\widehat{P}_{s,a}(s') - P_{s,a}(s') \big)\edited{\min(V^{\pi}(s'),x)} \right|.
\end{align}

We define the function $g$ as the difference in the expectation of $V$ between the two transition kernel for a fix $x$ and $(s,a)$,
\begin{align}
    g_{s,a}(x,V) \coloneqq \left|\sum_{s'} \big(\widehat{P}_{s,a}(s') - P_{s,a}(s') \big)\edited{\min\left( V^\pi(s'),x \right)} \right|.
\end{align}
Using Hoeffding's inequality, one has that with probability at least $1-\delta$,
\begin{align}\label{Eq:g_bound}
    g_{s,a}(x,V) \leq \sqrt{\frac{2\log(2/\delta)}{N(1-\gamma)^2}}.
\end{align}
It can be easily shown that $g_{s,a}(s,V)$ is 1-Lipschitz in $x$ for any V such that $\|V\|_{\infty} \leq \frac{1}{1-\gamma}$. To obtain the union bound, we construct an $\epsilon_1$-net $\mathcal{N}_{\epsilon_1}$ over $[0,\frac{1}{1-\gamma}]$, with size $|\mathcal{N}_{\epsilon_1}| \leq \frac{3}{(1-\gamma)\epsilon_1}$ \citep{vershynin2018high}.
By the union bound, we have that with probability at least $1-\delta$,
\begin{align}\label{Hoeffding bound}
    \sup_{x\in\mathbb{R}} \left|\sum_{s'} \big(\widehat{P}_{s,a}(s') - P_{s,a}(s') \big)\edited{\min\left( V^\pi(s'),x \right)} \right| & \overset{(i)}{\leq} \epsilon_1  + \sup_{x\in \mathcal{N}_{\epsilon_1}} \left|\sum_{s'} \big(\widehat{P}_{s,a}(s') - P_{s,a}(s') \big)\edited{\min\left( V^\pi(s'),x \right)} \right|  \nn \\ 
    & \overset{(ii)}{\leq} \sqrt{\frac{2\log(2SA|\mathcal{N}_{\epsilon_1}|/\delta)}{N(1-\gamma)^2}} + \epsilon_1 \nn \\
    & \overset{(iii)}{\leq} 2\sqrt{\frac{2\log(6SAN/\delta)}{N(1-\gamma)^2}},
\end{align}
where (i) follows from that the optimal $x$ falls into an $\epsilon_1$-ball centered around some point in $\mathcal{N}_{\epsilon_1}$ and $g_{s,a}(x,V)$ is 1-Lipschitz in $x$;
(ii) stems from applying the results in (\ref{Eq:g_bound}) and the union bound over $\mathcal{S}$, $\mathcal{A}$, and $\mathcal{N}_{\epsilon_1}$; and (iii) follows if we let $\epsilon_1 = \sqrt{\frac{2\log(6SAN/\delta)}{N(1-\gamma)^2}}$ and then $|\mathcal{N}_{\epsilon_1}| \leq \frac{3}{\epsilon_1(1-\gamma)} \leq 3N$.
Substituting (\ref{Hoeffding bound}) back into (\ref{point-wise difference}), we have that  with probability at least $ 1-\delta$,
\begin{equation}\label{point-wise concentration}
    \left|(\widehat{\vect{P}}_{s,a}^{\pi,V} - \vect{P}_{s,a}^{\pi,V})\vect{V}^{\pi}\right| \leq \frac{2}{\tau}\sqrt{\frac{2\log(6SAN/\delta)}{ N(1-\gamma)^2}}.
\end{equation}
From (\ref{point-wise concentration}), it can be shown that with probability at least $ 1-\delta$,
\begin{equation}\label{Eq:matrix_hoeffding}
    \left\|\underline{\widehat{\myMatrix{P}}}^{\pi,V}\vect{V}^{\pi} - \underline{\myMatrix{P}}^{\pi,V} \vect{V}^{\pi}\right\|_{\infty} \leq \frac{2}{\tau}\sqrt{\frac{2 \log (6SAN/\delta)}{N(1-\gamma)^2}}.
\end{equation}
This concludes the proof of Lemma \ref{Crude value bound lemma}.
\end{proof}
Substituting (\ref{Eq:matrix_hoeffding}) back into (\ref{same policy different environment error}) we get that with probability at least $ 1-\delta$,
\begin{align}\label{env error}
    \|\widehat{\vect{V}}^\pi - \vect{V}^\pi||_{\infty} &\overset{(i)}{\leq} \gamma \max \biggl\{\frac{2}{\tau}(\myMatrix{I} - \gamma\underline{\widehat{\myMatrix{P}}}^{\pi,\widehat{V}})^{-1}\left\|\underline{\widehat{\myMatrix{P}}}^{\pi,V}\vect{V}^{\pi} - \underline{\myMatrix{P}}^{\pi,V} \vect{V}^{\pi}\right\|_{\infty}\mathbf{1}, \nn \\  
    &  \qquad \frac{2}{\tau}(\myMatrix{I} - \gamma\underline{\widehat{\myMatrix{P}}}^{\pi,V})^{-1} \left\|\underline{\widehat{\myMatrix{P}}}^{\pi,V}\vect{V}^{\pi} - \underline{\myMatrix{P}}^{\pi,V} \vect{V}^{\pi}\right\|_{\infty}\mathbf{1} 
     \biggl\} \nn \\
    & \overset{(ii)}{\leq} \frac{2\gamma}{\tau(1-\gamma)}\sqrt{\frac{2\log(6SAN/\delta)}{N(1-\gamma)^2}},
\end{align}
where (i) holds by $\left(\myMatrix{I} - \gamma \myMatrix{\underline{\widehat{P}}}^{\pi,V}\right)^{-1} = \sum_{t=0}^\infty \gamma^t(\myMatrix{\underline{\widehat{P}}}^{\pi,V})^t\geq 0$, (ii) follows from
\begin{align}\label{Eq:(I-P)-1}
    \left(\myMatrix{I} - \gamma \myMatrix{\underline{\widehat{P}}}^{\pi,V}\right)^{-1}\mathbf{1} = \sum_{t=0}^\infty \gamma^t(\myMatrix{\underline{\widehat{P}}}^{\pi,V})^t \mathbf{1} = \frac{1}{1-\gamma} \mathbf{1}.
\end{align}
Finally take $\epsilon_{opt} \leq \frac{\sqrt{2\log(6SAN/\delta)}}{\tau(1-\gamma)\sqrt{N}}$, and 
plug (\ref{env error}) back to (\ref{decomposing the error}). We then have that with probability at least $1-\delta$, 
\begin{equation}\label{Eq:upper_bound_1}
    \begin{split}
        \|\vect{V}^* - \vect{V}^{\widehat{\pi}}\|_{\infty} &\leq  \|\vect{V}^{\pi^*} - \widehat{\vect{V}}^{\pi^*}\|_{\infty} + \frac{2\gamma\epsilon_{\text{opt}}}{1-\gamma} + \|\widehat{\vect{V}}^{\widehat{\pi}} - \vect{V}^{\widehat{\pi}}\|_{\infty} \\
        &\leq \frac{6\gamma}{\tau(1-\gamma)}\sqrt{\frac{2\log(6SAN/\delta)}{N(1-\gamma)^2}} \\
        &\leq 6\sqrt{\frac{2\log(6SAN/\delta)}{N(1-\gamma)^4\tau^2}}.
    \end{split}
\end{equation}
\subsection{Tighter bound using Berstein inequality when $\tau \geq \gamma$}
Note that the bound in (\ref{Eq:upper_bound_1}) applied to any $\tau$ and $\gamma$. In this subsection, we consider the scenario where $\tau \geq \gamma$. We show that a tighter bound can be achieved.

Then, for a fixed $x$ that is independent with $P_{s,a}$, using Bernstein inequality, one has that with probability at least $1-\delta$,
\begin{align}\label{Eq: g-func}
    g_{s,a}(x,V) &\leq \sqrt{\frac{2\log(\frac{2}{\delta})}{N}}\sqrt{\Var_{P_{s,a}}(\min(V,x))} + \frac{2\log\left(\frac{2}{\delta} \right)}{3N(1-\gamma)} \nn \\
    & \leq \sqrt{\frac{2\log(\frac{2}{\delta})}{N}}\sqrt{\Var_{P_{s,a}}(V)} + \frac{2\log\left(\frac{2}{\delta} \right)}{3N(1-\gamma)}.
\end{align}
To derive the union bound, we can similarly construct a $\epsilon_1$-net over $[0,\frac{1}{1-\gamma}]$ with size $|\mathcal{N}_{\epsilon_1}|\leq \frac{3}{\epsilon_1(1-\gamma)}$. By the union bound and (\ref{Eq: g-func}), it holds with probability at least $1-\frac{\delta}{SA}$ for all $x\in \mathcal{N}_{\epsilon_1}$,
\begin{align}
    g_{s,a}(x,V) \leq \sqrt{\frac{2\log(\frac{2SA |\mathcal{N}_{\epsilon_1}|}{\delta})}{N}}\sqrt{\Var_{P_{s,a}}(V)} + \frac{2\log\left(\frac{2SA |\mathcal{N}_{\epsilon_1}|}{\delta} \right)}{3N(1-\gamma)}.
\end{align}
From (\ref{point-wise difference}), we can show that with probability at least $1-\frac{\delta}{SA}$,
\begin{align}\label{Eq:Bersteinbound}
    \left|(\widehat{\vect{P}}_{s,a}^{\pi,V} - \vect{P}_{s,a}^{\pi,V})\vect{V}^{\pi} \right| &\leq \frac{1}{\tau} \sup_{x\in[0,\frac{1}{1-\gamma}]} \left|\sum_{s'} \big(\widehat{P}_{s,a}(s') - P_{s,a}(s') \big)\edited{\min\left( V^\pi(s'),x \right)} \right| \nn \\
    &\overset{(i)}{\leq} \frac{1}{\tau}\left( \epsilon_1 + \sqrt{\frac{2\log(\frac{2SA |\mathcal{N}_{\epsilon_1}|}{\delta})}{N}}\sqrt{\Var_{P_{s,a}}(V^{\pi})} + \frac{2\log\left(\frac{2SA |\mathcal{N}_{\epsilon_1}|}{\delta} \right)}{3N(1-\gamma)} \right) \nn \\
    &\overset{(ii)}{=} \frac{1}{\tau}\left( \sqrt{\frac{2\log(\frac{2SA |\mathcal{N}_{\epsilon_1}|}{\delta})}{N}}\sqrt{\Var_{P_{s,a}}(V^{\pi})} + \frac{\log\left(\frac{2SA |\mathcal{N}_{\epsilon_1}|}{\delta} \right)}{N(1-\gamma)} \right) \nn \\
    &\overset{(iii)}{\leq} \frac{1}{\tau}\left( 2\sqrt{\frac{\log(\frac{18SAN}{\delta})}{N}}\sqrt{\Var_{P_{s,a}}(V^{\pi})} + \frac{\log\left(\frac{18SAN}{\delta} \right)}{N(1-\gamma)} \right),
\end{align}
where (i) follows from that the optimal $x$ falls into an $\epsilon_1$-ball centered around some point in $\mathcal{N}_{\epsilon_1}$ and $g_{s,a}$ is 1-Lipschitz in $x$. (ii) stems from taking $\epsilon_1 = \frac{\log(2SA|\mathcal{N}_{\epsilon_1}|/\delta)}{3N(1-\gamma)}$; and $(iii)$ is shown by $|\mathcal{N}_{\epsilon_1}|\leq \frac{3}{\epsilon_1(1-\gamma)}\leq 9N$.

To bound the \edited{second} term on the right-hand side of (\ref{same policy different environment error}), for any policy $\pi$, we show that
\begin{align}\label{Eq:P&Var decompose}
    & \left(\myMatrix{I} - \gamma \myMatrix{\underline{\widehat{P}}}^{\pi,V}\right)^{-1} \left( \widehat{\underline{\myMatrix{P}}}^{\pi,V} \vect{V}^{\pi} - {\underline{\myMatrix{P}}}^{\pi,V} \vect{V}^{\pi}\right) \nn \\   
    & \overset{(i)}{\leq}  \left(\myMatrix{I} - \gamma \myMatrix{\underline{\widehat{P}}}^{\pi,V}\right)^{-1} \edited{\left| \widehat{\underline{\myMatrix{P}}}^{\pi,V} \vect{V}^{\pi} - {\underline{\myMatrix{P}}}^{\pi,V} \vect{V}^{\pi}\right|} \nn \\
    & \overset{(ii)}{\leq} \frac{\log\left(\frac{18SAN}{\delta}\right)}{N\tau(1-\gamma)}\left(\myMatrix{I} - \gamma \myMatrix{\underline{\widehat{P}}}^{\pi,V}\right)^{-1} \mathbf{1} + \underbrace {\frac{2}{\tau}\sqrt{\frac{\log\left(\frac{18SAN}{\delta} \right)}{N}} \left(\myMatrix{I} - \gamma \myMatrix{\underline{\widehat{P}}}^{\pi,V}\right)^{-1}\sqrt{\Var_{\underline{\widehat{\myMatrix{P}}}^{\pi,V}}(\vect{V}^{\pi})}}_{\eqqcolon \mathcal{C}_1} \nn \\
    & \quad + \underbrace {\frac{2}{\tau}\sqrt{\frac{\log\left(\frac{18SAN}{\delta} \right)}{N}} \left(\myMatrix{I} - \gamma \myMatrix{\underline{\widehat{P}}}^{\pi,V}\right)^{-1} \sqrt{ \left| \Var_{\widehat{\myMatrix{P}}^{\pi}}(\vect{V}^{\pi}) - \Var_{\underline{\widehat{\myMatrix{P}}}^{\pi,V}}(\vect{V}^{\pi})
    \right|}}_{\eqqcolon \mathcal{C}_2} \nn \\
    & \quad + \underbrace {\frac{2}{\tau}\sqrt{\frac{\log\left(\frac{18SAN}{\delta} \right)}{N}} \left(\myMatrix{I} - \gamma \myMatrix{\underline{\widehat{P}}}^{\pi,V}\right)^{-1} \left( \sqrt{\Var_{{\myMatrix{P}}^{\pi}}(\vect{V}^{\pi})} - \sqrt{\Var_{{\widehat{\myMatrix{P}}}^{\pi}}(\vect{V}^{\pi})}\right)}_{\eqqcolon \mathcal{C}_3},
\end{align}
where (i) holds by $\left(\myMatrix{I} - \gamma \myMatrix{\underline{\widehat{P}}}^{\pi,V}\right)^{-1} = \sum_{t=0}^\infty \gamma^t(\myMatrix{\underline{\widehat{P}}}^{\pi,V})^t\geq 0$; and (ii) holds with high probability by the bound in (\ref{Eq:Bersteinbound}) and 
\begin{align}
    \sqrt{\Var_{\myMatrix{P}^{\pi}}(\vect{V}^{\pi})} \leq \left(\sqrt{\Var_{\myMatrix{P}^{\pi}}(V^\pi)} - \sqrt{\Var_{\myMatrix{\widehat{P}}^{\pi}}(\vect{V}^{\pi})} \right) + \sqrt{\left| \Var_{\myMatrix{\widehat{P}}^{\pi}}(\vect{V}^{\pi}) - \Var_{\underline{\myMatrix{\widehat{P}}}^{\pi,V}}
    (\vect{V}^{\pi})\right|} + \sqrt{\Var_{\underline{\myMatrix{\widehat{P}}}^{\pi,V}}
    (\vect{V}^{\pi})}.
\end{align}
\begin{lemma}[\citet{Argawal_generative_model}, Lemma 4]\label{Lemma:Agarwal Lemma 4}
    For any policy $\pi$ and transition kernel $P$,
    \begin{align}
        \left\|\left( \myMatrix{I} - \gamma \myMatrix{P}^{\pi}\right)^{-1}\sqrt{\Var_{\myMatrix{P}^{\pi}}(\vect{V}^{\pi})} \right\|_{\infty} \leq \sqrt{\frac{2}{(1-\gamma)^3}}.
    \end{align}
\end{lemma}
Applying Lemma \ref{Lemma:Agarwal Lemma 4}, then $\mathcal{C}_1$ in (\ref{Eq:P&Var decompose}) can be bounded as follows
\begin{align}\label{Eq:C1}
    \mathcal{C}_1 = \frac{2}{\tau}\sqrt{\frac{\log\left(\frac{18SAN}{\delta} \right)}{N}} \left(\myMatrix{I} - \gamma \myMatrix{\underline{\widehat{P}}}^{\pi,V}\right)^{-1}\sqrt{\Var_{\underline{\widehat{\myMatrix{P}}}^{\pi,V}}(\vect{V}^{\pi})} \leq 2\sqrt{\frac{2\log\left(\frac{18SAN}{\delta} \right)}{N\tau^2(1-\gamma)^3}} \mathbf{1}.
\end{align}
\begin{lemma}\label{Lemma: CVaR TV bound}
Consider a CVaR uncertainty set $\mathcal{U}^\tau$ defined in (\ref{CVaR uncertainty set}), then it holds that
\begin{align}
    \forall \varient{\vect{P}}_{s,a} \in \mathcal{U}^{\tau}(\vect{P}_{s,a}): \|\vect{P}_{s,a} - \varient{\vect{P}}_{s,a}\|_1 \leq  \frac{2(1-\tau)}{\tau}.
\end{align}
\end{lemma}
Then we have for all $(s,a)\in \mathcal{S} \times \mathcal{A}$, and $\varient{P}_{s,a}\in \mathcal{U}^{\tau}(P_{s,a})$:
\begin{align}\label{Eq:c_2bound}
    \left| \Var_{\varient{P}_{s,a}}(V^\pi) - \Var_{P_{s,a}}(V^\pi)\right| \leq 3\left\|\varient{P}_{s,a} - P_{s,a}\right\|_1 \left\|V^\pi\right\|^2_{\infty} \leq 6\frac{1-\tau}{\tau} \frac{1}{(1-\gamma)^2}  \overset{(i)}{\leq} \frac{6}{\tau(1-\gamma)},
\end{align}
and (i) holds when $\tau \geq \gamma$. We then have that
\begin{align}\label{Eq:C2}
    \mathcal{C}_2 &= \frac{2}{\tau}\sqrt{\frac{\log\left(\frac{18SAN}{\delta} \right)}{N}} \left(\myMatrix{I} - \gamma \myMatrix{\underline{\widehat{P}}}^{\pi,V}\right)^{-1} \sqrt{ \left| \Var_{\widehat{\myMatrix{P}}^{\pi}}(\vect{V}^{\pi}) - \Var_{\underline{\widehat{\myMatrix{P}}}^{\pi,V}}(\vect{V}^{\pi})
    \right|} \nn \\
    & = \frac{2}{\tau}\sqrt{\frac{\log\left(\frac{18SAN}{\delta} \right)}{N}} \left(\myMatrix{I} - \gamma \myMatrix{\underline{\widehat{P}}}^{\pi,V}\right)^{-1} \sqrt{ \left|\myMatrix{\Pi}^\pi\left( \Var_{\widehat{\myMatrix{P}}}(\vect{V}^{\pi}) - \Var_{{\widehat{\myMatrix{P}}}^{\pi,V}}(\vect{V}^{\pi})\right)
    \right|} \nn \\
    & \leq \frac{2}{\tau}\sqrt{\frac{\log\left(\frac{18SAN}{\delta} \right)}{N}} \left(\myMatrix{I} - \gamma \myMatrix{\underline{\widehat{P}}}^{\pi,V}\right)^{-1} \sqrt{\left\| \Var_{\widehat{\myMatrix{P}}}(\vect{V}^{\pi}) - \Var_{{\widehat{\myMatrix{P}}}^{\pi,V}}(\vect{V}^{\pi})\right\|_{\infty}} \mathbf{1}\nn \\
    & \overset{(i)}{\leq} \frac{2}{\tau}\sqrt{\frac{\log\left(\frac{18SAN}{\delta} \right)}{N}} \left(\myMatrix{I} - \gamma \myMatrix{\underline{\widehat{P}}}^{\pi,V}\right)^{-1} \sqrt{\frac{6(1-\tau)}{\tau(1-\gamma)^2}}\mathbf{1} \nn \\
    & \overset{(ii)}{=} 2\sqrt{\frac{6(1-\tau)\log\left(\frac{18SAN}{\delta} \right)}{N\tau^3(1-\gamma)^4}}\mathbf{1},
\end{align}
where (i) follows from (\ref{Eq:c_2bound}); and (ii) follows from (\ref{Eq:(I-P)-1}).

In the following, we then bound $\mathcal{C}_3$. The following lemma plays an important role.
\begin{lemma}[\citealp{panaganti2022sample}, Lemma 6]\label{Lemma:panaganti}
    Consider any $\delta \in (0,1)$. For any fixed policy $\pi$ and fixed value function vector $V\in \mathbb{R}^{S}$, one has that with probability at least $1-\delta$,
    \begin{align}\label{Eq:lemma6}
        \left| \sqrt{\Var_{\widehat{\myMatrix{P}}^\pi}(V)} - \sqrt{\Var_{{\myMatrix{P}}^\pi}(V)} \right| \leq \sqrt{\frac{2\|V\|_{\infty}^2 \log(\frac{2SA}{\delta})}{N}}\mathbf{1}.
    \end{align}
\end{lemma}
Plugging (\ref{Eq:lemma6}) into (\ref{Eq:P&Var decompose}), we can show that with high probability
\begin{align}\label{Eq:C3}
    \mathcal{C}_3 & = \frac{2}{\tau}\sqrt{\frac{\log\left(\frac{18SAN}{\delta} \right)}{N}} \left(\myMatrix{I} - \gamma \myMatrix{\underline{\widehat{P}}}^{\pi,V}\right)^{-1} \left( \sqrt{\Var_{{\myMatrix{P}}^{\pi}}(\vect{V}^{\pi})} - \sqrt{\Var_{{\widehat{\myMatrix{P}}}^{\pi}}(\vect{V}^{\pi})}\right) \nn \\
    & \leq \frac{4}{1-\gamma} \frac{\log\left( \frac{18SAN}{\delta}\right)\|\vect{V}^{\pi}\|_{\infty}}{N\tau}\mathbf{1} \leq \frac{4\log\left( \frac{18SAN}{\delta}\right)}{N\tau(1-\gamma)^2}\mathbf{1}.
\end{align}
Finally, plugging the results of $\mathcal{C}_1$ in (\ref{Eq:C1}), $\mathcal{C}_2$ in (\ref{Eq:C2}) and $\mathcal{C}_3$ in (\ref{Eq:C3}) back into (\ref{Eq:P&Var decompose}), we have
\begin{align}\label{Eq:vectorized_bern}
    &\left(\myMatrix{I} - \gamma \myMatrix{\underline{\widehat{P}}}^{\pi,V}\right)^{-1} \left( \widehat{\underline{\myMatrix{P}}}^{\pi,V} \vect{V}^{\pi} - {\underline{\myMatrix{P}}}^{\pi,V} \vect{V}^{\pi}\right)  \leq 2\sqrt{\frac{2\log\left(\frac{18SAN}{\delta} \right)}{N\tau^2(1-\gamma)^3}} \mathbf{1} + 2\sqrt{\frac{6(1-\tau)\log\left(\frac{18SAN}{\delta} \right)}{N\tau^3(1-\gamma)^4}}\mathbf{1} \nn \\
    & \qquad + \frac{4\log\left( \frac{18SAN}{\delta}\right)}{N\tau(1-\gamma)^2}\mathbf{1} + 
    \frac{\log\left(\frac{18SAN}{\delta}\right)}{N\tau(1-\gamma)^2} \mathbf{1} \nn \\
    &= \frac{5\log\left(\frac{18SAN}{\delta}\right)}{N\tau(1-\gamma)^2} \mathbf{1} + \sqrt{\frac{\log\left(\frac{18SAN}{\delta} \right)}{N\tau^2(1-\gamma)^3}}\left(2\sqrt{2} + 2\sqrt{6\frac{1-\tau}{\tau(1-\gamma)}} \right)\mathbf{1} \nn \\
    & \overset{(i)}{\leq} \frac{10\log\left(\frac{18SAN}{\delta}\right)}{N(1-\gamma)^2} \mathbf{1} + 32\sqrt{\frac{\log\left(\frac{18SAN}{\delta} \right)}{N(1-\gamma)^3}}\mathbf{1},
\end{align}
where (i) holds when $\tau \geq \gamma \geq \frac{1}{2}$.

Plugging (\ref{Eq:vectorized_bern}) \edited{and \eqref{Eq:second_upper_result}} back into (\ref{same policy different environment error}), we have that
\begin{align}
    \left\| \vect{V}^{\pi} - \widehat{V}^{\pi} \right\|_{\infty} \leq \frac{10\log\left(\frac{18SAN}{\delta}\right)}{N(1-\gamma)^2}  + 32\sqrt{\frac{\log\left(\frac{18SAN}{\delta} \right)}{N(1-\gamma)^3}}.
\end{align}
Summing up the results for $\widehat{\pi}$ and $\pi^*$ and plugging back to (\ref{decomposing the error}) complete the proof as follows: taking $\epsilon_{\text{opt}}\leq \frac{\log(\frac{18SAN}{\delta})}{\gamma (1-\gamma)N}$ and $N \geq \frac{\log(\frac{18SAN}{\delta})}{(1-\gamma)^2}$, with probability at least $1-\delta$,
\begin{align}
    \left\| \vect{V}^* - \vect{V}^{\widehat{\pi}} \right\|_{\infty} &\leq 
    \frac{2\gamma \epsilon_{\text{opt}}}{1-\gamma} + \left\| \widehat{\vect{V}}^{\widehat{\pi}} - \vect{V}^{\widehat{\pi}} \right\|_{\infty} + \left\| \vect{V}^{\pi^*} - \widehat{\vect{V}}^{\pi^*} \right\|_{\infty} \nn \\
    & \leq  \frac{2\gamma \epsilon_{\text{opt}}}{1-\gamma} + \frac{20\log\left(\frac{18SAN}{\delta}\right)}{N(1-\gamma)^2} + 64\sqrt{\frac{\log\left(\frac{18SAN}{\delta}\right)}{N(1-\gamma)^3}} \nn \\
    & \leq \frac{22\log\left(\frac{18SAN}{\delta}\right)}{N(1-\gamma)^2} + 64\sqrt{\frac{\log\left(\frac{18SAN}{\delta}\right)}{N(1-\gamma)^3}} \nn \\
    & \leq 86\sqrt{\frac{\log\left(\frac{18SAN}{\delta}\right)}{N(1-\gamma)^3}},
\end{align}
where the last inequality holds if $N \geq \frac{\log(\frac{18SAN}{\delta})}{(1-\gamma)^2}$, with probability at least $1-\delta$.

\section{Proof of Theorem \ref{Theorem lower bound}: Sample Complexity Lower Bound}

\subsection{Construction of hard problem instances}

\textbf{Construction of two hard MDPs.} Suppose there are two standard MDPs defined as below:
\[\mathcal{M_{\phi}} = \{ (\mathcal{S}, \mathcal{A}, P^{\phi}, r, \gamma) | \phi = {0,1} \}. \] 
Here, $\gamma$ is the discount factor, $\mathcal{S} = \{0,1,...,S-1\}$ is the state space. Given any state $s \in \{ 2, 3, ..., S-1\}$, the coresponding action spaces are $\mathcal{A} = \{0,1,2,...,A-1\}$. While for state $s = 0$ and $s = 1$, the action space is only $\mathcal{A'} = \{0,1\}$. For any $\phi \in \{0,1\}$, the transition kernel $P^{\phi}$ of the constructed MDP $\mathcal{M}_{\phi}$ is defined as
\begin{equation}\label{Eq:CVaR_transition_lower_bound}
P^{\phi}(s'|s,a) = 
     \begin{cases}
       p\mathds{1}(s' = 1) + (1-p)\mathds{1}(s' = 0), &\quad\text{if} \qquad (s,a) = (0,\phi)\\
       q\mathds{1}(s' = 1) + (1-q)\mathds{1}(s' = 0), &\quad\text{if} \qquad (s,a) = (0,1-\phi)\\
       \mathds{1}(s' = 1), &\quad\text{if} \qquad s \geq 1\\
     \end{cases},
\end{equation}
where $p$ and $q$ are set to satisfy
\begin{equation}
    0\leq p\leq 1 \qquad \text{and} \qquad 0 \leq q = p-\Delta\label{definition of p&q}
\end{equation} 
for some $p$ and $\Delta > 0$. The above transition kernel $P^{\phi}$ implies that State $1$ is an absorbing state.

Then we define the reward function as:
\begin{align}
    r(s,a) = \begin{cases}
        1, & \quad \text{if} \quad $s = 1$ \nn \\
        0, & \quad \text{otherwise}.
    \end{cases}
\end{align}
Additionally, we choose the following initial state distribution:
\begin{align}
    \varphi(s) = \begin{cases}
        1, & \quad \text{if} \quad $s = 0$\nn \\
        0, & \quad \text{otherwise}.
    \end{cases}
\end{align}

\textbf{Uncertainty set of the transition kernel.} Recall that the uncertainty set of CVaR is:
\begin{equation}
    \mathcal{U}^{\tau}(\vect{P}_{s,a}) = \left\{ \varient{\vect{P}}_{s,a} \in \Delta(\mathcal{S}), \ 0\leq \frac{\varient{P}_{s,a}(s')}{P_{s,a}(s')}\leq \frac{1}{\tau} \right\}.
\end{equation}
We define $p$ and $\Delta$:
\begin{equation}
    p = (1-\tau) + c \tau (1-\gamma) \quad and \quad \Delta \leq c\tau(1-\gamma),
\end{equation}
where $c\in (0,1)$. Consequently, applying (${\ref{definition of p&q}}$) directly leads to
\begin{equation}
    p \geq q \geq 1-\tau. \label{p&q geq 1-tau}
\end{equation}

For any $(s,a) \in \mathcal{S\times A}$, we denote the smallest transition probability of moving to the next state $s'\in \{0,1\}$ in the uncertainty set as
\begin{subequations}
\begin{align}\label{Eq:smallest_transition_p}
        \underline{P}^{\phi}(1|s,a) \coloneqq \inf_{P_{s,a}\in \mathcal{U}^{\tau}(P_{s,a}^{\phi})}P(1|s,a) &= 1 - \min\left\{\frac{1}{\tau}P_{s,a}^{\phi}(0),1 \right\} =  \frac{1}{\tau}\max\left\{P_{s,a}^{\phi}(1) - (1-\tau),0 \right\},  \\
        \label{Eq:smallest_transition_q}
        \underline{P}^{\phi}(0|s,a) \coloneqq \inf_{P_{s,a}\in \mathcal{U}^{\tau}(P_{s,a}^{\phi})}P(0|s,a) &= \min\left\{\frac{1}{\tau}P_{s,a}^{\phi}(0),1 \right\},
\end{align}
\end{subequations}
where the last equation can be verified by the definition of $\mathcal{U}^{\tau}(P)$. We further define the following notation
\begin{equation}\label{Eq:worst p&q}
    \underline{p} \coloneqq \underline{P}^{\phi}(1|0,\phi) = \frac{1}{\tau}\max\{p-(1-\tau),0\}, \quad  \underline{q} \coloneqq \underline{P}^{\phi}(1|0,1-\phi) = \frac{1}{\tau}\max\{q-(1-\tau),0\},
\end{equation}
which follows from the fact that $p\geq q \geq 1- \tau$ in (${\ref{p&q geq 1-tau}}$).

\textbf{Robust value function and optimal robust policies.}
For any MDP $\mathcal{M}^{\phi}$ with the above uncertainty set, we denote $\pi_{\phi}^*$ as the optimal policy, and the robust value function of any policy $\pi$ as $V_{\phi}^{\pi}$. Then we introduce the following lemma, which describes some important properties of the robust value function:

\begin{lemma}\label{lower bound value function lemma}
For any $\phi = \{0,1\}$ and any policy $\pi$, the value function satisfies
\begin{align}
    V_{\phi}^{\pi}(0) = \frac{\gamma z_{\phi}^{\pi}}{(1-\gamma)(1-\gamma(1-z_\phi^\pi))}, \label{lower bound value function}
\end{align}
where $z_\phi^\pi$ is defined as
\begin{align}
    z_{\phi}^{\pi} &= 1-\frac{1}{\tau}(1 - p\pi(\phi|0)- q\pi(1-\phi|0)) \nn \\
    & = \underline{p}\pi(\phi|0) + \underline{q}\pi(1-\phi|0).
\end{align}
The optimal value function and the optimal policy satisfy
\begin{subequations}
\begin{align}
    V_{\phi}^*(0) &= \frac{\gamma \underline{p}}{(1-\gamma)(1-\gamma(1-\underline{p}))},\label{lower bound optimal value function}\\
    \pi_{\phi}^*(\phi|s) &= 1, \quad \text{for}\ s \in \mathcal{S}.
\end{align}
\end{subequations}
\end{lemma}

\subsection{Establishing the minimax lower bound}\label{appendix lower bound}
Note that our goal is to quantify the sub-optimality gap between the policy estimator $\widehat{\pi}$ and the optimal policy $\pi^*$ on the initial state distribution $\varphi$, which is
\[\langle \varphi , V_{\phi}^* - V_{\phi}^{\widehat{\pi}}\rangle = V_{\phi}^* (0)- V_{\phi}^{\widehat{\pi}}(0).\]

\textbf{Step 1: Equivalence to estimating $\phi$.}
With $\epsilon \leq \frac{c}{16(1-\gamma)}$, let $\Delta = 16(1-\gamma)^2\tau \epsilon$. Applying (\ref{lower bound value function}) and (\ref{lower bound optimal value function}) we have that
\begin{align} 
\label{connection between sub optimality gap and policy}
        V_{\phi}^* (0)- V_{\phi}^{\widehat{\pi}}(0) &= \frac{\gamma \underline{p}}{(1- \gamma)(1-\gamma(1-\underline{p}))} - \frac{\gamma z_\phi^{\widehat{\pi}}}{(1-\gamma)(1-\gamma(1-z_\phi^{\widehat{\pi}}))} \nn \\ 
        & = \frac{\gamma(\underline{p} - z_\phi^{\widehat{\pi}})}{(1-\gamma(1-\underline{p}))(1-\gamma(1-z_\phi^{\widehat{\pi}}))} \nn \\
        & = \frac{\gamma(\underline{p} - \underline{q})}{(1-\gamma)^2\left(1+\frac{\gamma \underline{p}}{1-\gamma}\right)\left(1 + \frac{\gamma z_\phi^{\widehat{\pi}}}{1-\gamma}\right)} (1-{\widehat{\pi}}(\phi|0)) \nn \\
        & \overset{\text{(i)}}{=} \frac{\gamma \Delta}{\tau(1-\gamma)^2\left(1+\frac{\gamma \underline{p}}{1-\gamma}\right)\left(1 + \frac{\gamma z_\phi^{\widehat{\pi}}}{1-\gamma}\right)} (1-{\widehat{\pi}}(\phi|0)) \nn \\
        & \overset{\text{(ii)}}{\geq} \frac{\gamma}{(1-\gamma)^2}\frac{1}{\left(1+\frac{\gamma \underline{p}}{1-\gamma}\right)^2}\left(\frac{\Delta}{\tau}\right)(1-{\widehat{\pi}}(\phi|0)) \nn\\
        & \overset{\text{(iii)}}{=} \frac{\gamma}{\tau(1-\gamma^2)} \frac{1}{(1+c\gamma)^2}\Delta (1-{\widehat{\pi}}(\phi|0)) \nn \\
        & \overset{\text{(iv)}}{\geq} 2\epsilon (1 - {\widehat{\pi}}(\phi|0)),
\end{align}
where (i) holds by the definition of $\underline{p} $ and $\underline{q}$; (ii) follows from $z_{\phi}^{\widehat{\pi}} \leq \underline{p}$; (iii) follows from the definition of $\underline{p}$; (iv) follows from $\gamma \geq \frac{1}{2}$.

With this connection between the sub-optimality gap and the policy ${\widehat{\pi}}(\phi|0)$, if the policy ${\widehat{\pi}}$ is $\epsilon$-optimal with a high probability, i.e.,
\begin{equation}
    P \left(\langle \varphi,V_{\phi}^* - V_\phi^{\widehat{\pi}} \rangle \leq \epsilon \right) \geq 1-\delta,
\end{equation}
then, we need ${\widehat{\pi}}(\phi|0)\geq \frac{1}{2}$ with probabily at least $1-\delta$. With this in mind, we can construct an estimator $\widehat{\phi}$
 for the better action $\phi$
 \begin{equation}
     \label{definition of phi hat}
     \widehat{\phi} = \argmax_{a \in \{0,1\}} {\widehat{\pi}}(a|0),
 \end{equation}
which satisfies 
\begin{equation}
\label{High prob to distinguish}
    P\left(\widehat{\phi} = \phi \right) = P\left(\widehat{\pi}(\phi|0) > \frac{1}{2}\right) \geq 1-\delta.
\end{equation}
The problem now becomes to produce a correct estimator $\widehat{\phi}$ with high probability.
Subsequently, the goal is to demonstrate that (\ref{High prob to distinguish}) cannot hold without a sufficient number of samples.

\textbf{Step 2: Probability of error in testing two hypotheses.}
Equipped with the aforementioned ground-
work, we can now delve into differentiating between the two hypotheses $\phi \in \{0,1\}$. To achieve this, we
consider the concept of minimax probability of error, defined as follows:
\begin{equation}
\label{minimax probability of error}
    p_e \coloneqq \inf_{\psi} \max \{P_0(\Psi \neq 0),P_1(\Psi \neq 1)\}.
\end{equation}
Here, the infimum is taken over all possible tests $\Psi$ constructed from the samples generated from the nominal
transition kernel $P^\phi$.

Moving forward, let us denote $\mu_\phi$ (resp. $\mu_\phi$(s)) as the distribution of a sample tuple $(s_i,a_i,s_
i')$ under the
nominal transition kernel $P^\phi$ associated with $\mathcal{M}_{\phi}$ and the samples are generated independently. Applying
standard results from \citet{chagny2016introduction}  and the additivity of the KL divergence, we obtain
\begin{align}
    \label{error & kl}
    p_e &\geq \frac{1}{4} \exp{\left(-NSA\cdot\text{KL}\left(\mu_0||\mu_1\right)\right)} \nn \\
    & = \exp{\big \{ -N \big(\text{KL}(P^0(\cdot|0,0)||P^1(\cdot|0,0)) + \text{KL}(P^0(\cdot|0,1)||P^1(\cdot|0,1))\big)  \big \} },
\end{align}
where the last inequality holds by observing that
\begin{align}
    \text{KL}(\mu_0||\mu_1) &= \frac{1}{SA} \sum_{s',a',s'}\text{KL}\left(P^0(s'||s,a)||P^1(s'||s,a)\right) \nn \\
     &= \frac{1}{SA} \sum_{a \in \{0,1\}}\text{KL}\left(P^0(s'||0,a)||P^1(s'||0,a)\right).
\end{align}
Now our focus shifts to bound the KL divergence in (\ref{error & kl}). Applying Lemma 2.7 in \citet{iyengar2005robust} gives
\begin{align}
\label{KL upper bound}
    \text{KL}\left(P^0(s'||0,1)||P^1(s'||0,1)\right) & = \text{Kl}(p||q)\leq \frac{(p - q)^2}{p(1-p)} = \frac{\Delta^2}{p(1-p)} \nn \\
    & = \frac{256(1-\gamma)^4\tau^2\epsilon^2}{p(1-p)}\nn \\
    &\overset{(i)}{\leq} \frac{256(1-\gamma)^4\tau^2\epsilon^2}{\gamma \tau (1 - \gamma \tau)} \nn \\
    &\overset{(ii)}{\leq} \frac{512(1-\gamma)^4\tau\epsilon^2}{1-\gamma \tau},
\end{align}
where (i) stems from $1-p = \tau - c\tau (1 - \gamma) = \tau(1-c(1-\gamma))\geq \gamma \tau$, and $p(1-p) \geq \min \{\tau (1 - \tau), \gamma \tau (1 - \gamma \tau) \} = \gamma \tau (1 - \gamma \tau)$; and (ii) follows from $\gamma \geq \frac{1}{2}$. Note that $\text{KL} \left(P^0(s'||0,0)||P^1(s'||0,0)\right)$ can be bounded in the same procedure. Substitute (\ref{KL upper bound}) back in to (\ref{error & kl}), if the sample size is selected as
\begin{equation}
\label{smallest sample size}
    N \leq \frac{\log\left( \frac{1}{4\delta} \right)(1 - \gamma \tau)}{1024(1-\gamma)^4\tau\epsilon^2},
\end{equation}
then one necessarily has
\begin{equation}
    p_e \geq \frac{1}{4}\exp \left\{ -N \cdot \frac{1024(1-\gamma)^4\tau\epsilon^2}{1-\gamma \tau} \right\} \geq \delta.
\end{equation}

\textbf{Step 3: putting the results together.} Combined all the results above, suppose there is a policy ${\widehat{\pi}}$ such that 
\begin{align}
    P_0\{\langle \varphi, V_0^* - V_0^{\widehat{\pi}} > \epsilon\rangle\} < \delta \qquad and \qquad P_1\{\langle \varphi, V_1^* - V_1^{\widehat{\pi}} > \epsilon\rangle\} < \delta.
\end{align}
According to the discussion in Step 1, the estimated $\widehat{\phi}$ must satisfy 
\begin{align}
    P_0(\widehat{\phi}\neq 0 )< \delta \qquad and \qquad P_1(\widehat{\phi}\neq 1) <\delta.
\end{align}
However, this cannot be satisfied with the sample size being too small (\ref{smallest sample size}). Thus we have completed the proof of the lower bound.

\section{Proof for Worst Path RL}
\subsection{Proof of Theorem \ref{theorem:worst-path upper bound}}
For any $(s, a) \in \mathcal{S} \times \mathcal{A}$, $n(s',s, a)$ denotes the number of times when the next state is $s'$. Using Lemma F.4 in \citet{dann2017unifying}, we have that
\begin{align}\label{Eq:Dann}
    P\left(n(s',s,a) \geq \frac{1}{2}NP(s'|s,a) - 2\log\left( \frac{SA}{\delta} \right) ,\ \forall (s,a) \in \mathcal{S} \times \mathcal{A}\right) \geq 1-\delta.
\end{align}
When $N \geq \frac{2}{p_{\min}}\left(1 + 2\log\left( \frac{SA}{\delta} \right) \right)$, for any $(s, a) \in \mathcal{S} \times \mathcal{A}$ and $s' \in \text{supp}\left( P(\cdot|s,a)\right)$, we can show that
\begin{align}
    &P\left(n(s',s,a) \geq \frac{1}{2}NP(s'|s,a) - 2\log\left( \frac{SA}{\delta} \right) \right) \nn \\
    &\leq P\left(n(s',s,a) \geq \frac{P(s'|s,a)}{p_{\min}}\left(1 + 2\log\left( \frac{SA}{\delta} \right) \right) - 2\log\left( \frac{SA}{\delta} \right) \right) \overset{(i)}{\leq} P\left(n(s',s,a) \geq 1 \right),
\end{align}
where (i) is derived from $P(s'|s,a) \geq p_{\min}$ for $s' \in \text{supp}\left( P(\cdot|s,a)\right)$. Combining with (\ref{Eq:Dann}), we have that
\begin{align}
    \forall s' \in \text{supp}\left( P(\cdot|s,a)\right), \ P\left(n(s',s,a) \geq 1 \right) \geq 1- \delta.
\end{align}
We decompose the sub-optimality gap in the same manner as in (\ref{decomposing the error}). We notice that with probability at least $1-\delta$, for all $(s, a) \in \mathcal{S} \times \mathcal{A}$,
\begin{align}\label{Eq:WorstPV}
    (\widehat{\vect{P}}_{s,a}^{\pi,V} - \vect{P}_{s,a}^{\pi,V})\vect{V}^{\pi}  = \min_{s'\in \text{supp}(\widehat{P}(\cdot|s,a))}V^{\pi}(s') - \min_{s'\in \text{supp}({P}(\cdot|s,a))}V^{\pi}(s') = 0.
\end{align}
Plugging (\ref{Eq:WorstPV}) back into (\ref{same policy different environment error}), we have
\begin{align}
    \left\|\widehat{\vect{V}}^{\pi} - \vect{V}^{\pi} \right\|_{\infty} = 0.
\end{align}
Then the first and the third term in (\ref{decomposing the error}) then disappears, and we get that
\begin{align}
    \|V^{*}-V^{\widehat{\pi}}\|_{\infty} \leq \frac{2\gamma \epsilon_{\text{opt}}}{1-\gamma}.
\end{align}
When $N \geq \frac{2}{p_{\min}}\left(1 + 2\log\left( \frac{SA}{\delta} \right) \right)$, for any $\epsilon > 0 $, if we take $\epsilon_{opt} \leq \frac{\epsilon(1-\gamma)}{2\gamma}$, then with probability at least $1-\delta$, we have that
\begin{align}
    \|\vect{V}^{*}-\vect{V}^{\widehat{\pi}}\|_{\infty} \leq \epsilon,
\end{align}
which concludes the proof.

\subsection{Proof of Theorem \ref{theorem:Worst-Path RL Lower Bound}}
In the following, we establish a sample complexity lower bound for Worst Path RL. We first construct two hard MDP instances similar to those in Section \ref{appendix lower bound}.
Suppose there are two standard MDPs defined below
\[\mathcal{M_{\phi}} = \{ (\mathcal{S}, \mathcal{A}, P^{\phi}, r, \gamma) | \phi = {0,1} \}.\] 
The transition kernel is defined as
\begin{align}
P^{\phi}(s'|s,a) = 
     \begin{cases}
       \mathbf{1}(s' = 1), &\quad\text{if} \qquad (s,a) = (0,\phi) \\
       (1-p_{\min})\mathbf{1}(s' = 0) + p_{\min}\mathbf{1}(s' = 1), &\quad\text{if} \qquad (s,a) = (0,1-\phi)  \\
       \mathbf{1}(s' = 1), &\quad\text{if} \qquad s \geq 1,
     \end{cases}
\end{align}
The reward function is defined as
\begin{align}
    r(s,a) = \begin{cases}
        1, & \quad \text{if} \quad $s = 1$\\
        0, & \quad \text{otherwise}.
    \end{cases}
\end{align}
The initial state distribution is also the same
\begin{equation}
    \varphi(s) = \begin{cases}
        1, & \quad \text{if} \quad $s = 0$\\
        0, & \quad \text{otherwise}.
    \end{cases}
\end{equation}
Since state $s = 1 $ is an absorbing state and has reward $1$, the value function at state 1 for any policy $\pi$ is $\vect{V}^{\pi}(1) = \frac{1}{1-\gamma}$. At state $s \in \{2,3,...,S-1\}$, applying the Bellman operator we have $\vect{V}^{\pi} (s)= \frac{\gamma}{1-\gamma} $. At state $s = 0$, we have that
\begin{align}
    V^{\pi}_{\phi}(0) &= \frac{\gamma}{1-\gamma} \pi(\phi|0),\\
    V^{* }_{\phi}(0) - V^{\pi}_{\phi}(0) &= \frac{\gamma}{1-\gamma} \left(1- \pi(\phi|0)\right).
\end{align}
If we want a policy $\pi$ to be $\epsilon$-optimal with high probability:
\begin{align}
    P\left\{ \langle \varphi, V_{\phi}^* - V_{\phi}^\pi \rangle  \leq \epsilon \right\} \geq 1-\delta,
\end{align}
then if $\epsilon \leq \frac{\gamma}{2(1-\gamma)}$, we necessarilly need $\pi(\phi|0) \geq \frac{1}{2}$ with probability at least $1 - \delta$. Following the same procedure in the proof of the lower bound for Iterated CVaR RL, we constructed the estimator $\widehat{\phi}$ for the better action (\ref{definition of phi hat}) that satisfies (\ref{High prob to distinguish}). We also notice that if $n(0,0,a)\geq 1$ in the $N$ samples, we can definitely tell the other action is superior. Otherwise, we cannot tell the difference between the two actions, and in this case, the probability of a correct guess is $\frac {1}{2}$. With this in mind, we have
\begin{align}\label{Eq:correct guess}
    P\left\{ \widehat{\phi} =\phi \right\} = 1-\left(1-p_{\min}\right)^N + \left(1-p_{\min}\right)^N \frac{1}{2}.
\end{align}
Inserting (\ref{Eq:correct guess}) into \ref{High prob to distinguish} we have that
\begin{align}
    1-\delta \leq P\left\{ \widehat{\phi} =\phi \right\} = 1-\frac{1}{2}\left(1-p_{\min}\right)^N,
\end{align}
which further implies that
\begin{align}
    N \geq \frac{\log(1/2\delta)}{\log(1/(1-p_{\min}))} \sim \mathcal{\tilde{O}}\left( \frac{1}{p_{\min}}\right).
\end{align}
The overall sample complexity lower bound is then $\mathcal{\tilde{O}}\left( \frac{SA}{p_{\min}}\right)$.

\section{Proof of Lemmas}
\subsection{Proof of Lemma \ref{Lemma: CVaR TV bound}}
    For any $\varient{\vect{P}}_{s,a} \in \mathcal{U}^{\tau}(\vect{P}_{s,a}) = \left\{ \varient{\vect{P}}_{s,a} \in \Delta(\mathcal{S}), \ 0\leq \frac{\varient{P}_{s,a}(s')}{P_{s,a}(s')}\leq \frac{1}{\tau} \right\}$, the total-variation between $\varient{\vect{P}}_{s,a}$ and $\vect{P}_{s,a}$ is defined as:
\begin{align}
    \left\|\varient{\vect{P}}_{s,a} - \vect{P}_{s,a} \right\|_{TV} = \sup_{\mathcal{S}'\subset\mathcal{S}} \left|\varient{P}_{s,a}(\mathcal{S}') - P_{s,a}(\mathcal{S}') \right| = \frac{1}{2}\left\|\varient{\vect{P}}_{s,a} - \vect{P}_{s,a} \right\|_1.
\end{align}
Here $\mathcal{S}'$ is a subset of state space $\mathcal{S}$ and $P_{s,a}(\mathcal{S}') = \sum_{s'\in \mathcal{S}'}P_{s,a}(s')$, $\varient{P}_{s,a}(\mathcal{S}') = \sum_{s'\in \mathcal{S}'}\varient{P}_{s,a}(s')$.

By the definition of the uncertainty set (\ref{CVaR uncertainty set}), for any $\mathcal{S}' \in \mathcal{S}$, we have that
\begin{align}
    0 \leq \varient{P}_{s,a}(\mathcal{S}') &\leq \frac{1}{\tau}P_{s,a}(\mathcal{S}'), \\
    0\leq1 - \varient{P}_{s,a}(\mathcal{S}') &\leq \frac{1}{\tau}\left( 1 - P_{s,a}(\mathcal{S}')\right).
\end{align}
This further implies that
\begin{align}
    1 - \frac{1}{\tau}\left( 1 - P_{s,a}(\mathcal{S}')\right) \leq \varient{P}_{s,a} (\mathcal{S}') \leq \frac{1}{\tau}P_{s,a}(\mathcal{S}').
\end{align}
For any $\mathcal{S}' \in \mathcal{S}$,
\begin{align}
    \left| \varient{P}_{s,a} (\mathcal{S}') - P_{s,a}(\mathcal{S}') \right| &\leq \max \left\{ \left( \frac{1}{\tau} - 1\right) P_{s,a}(\mathcal{S}'), \left( \frac{1}{\tau} - 1\right) (1-P_{s,a}(\mathcal{S}')) \right\} \nn \\
    &= \left( \frac{1}{\tau} - 1\right)\max \left\{  P_{s,a}(\mathcal{S}'), 1-P_{s,a}(\mathcal{S}') \right\} \nn \\
    & \leq \frac{1}{\tau} - 1.
\end{align}
Using the second definition of total variation, we obtain that
\begin{align}
    \frac{1}{2}\left\|\varient{\vect{P}}_{s,a} - \vect{P}_{s,a} \right\|_1 \leq \frac{1}{\tau} - 1.
\end{align}

\subsection{Proof of Lemma \ref{lower bound value function lemma}}
For any $\mathcal{M}_{\phi} \in \{0,1\}$, we can easily find the value function at state $s = 1$ or any states $s \in \{2,3,\cdots,S-1\}$
\begin{subequations}
    \begin{gather}
        V_\phi^\pi(1) \overset{(i)}{=} 1 + \gamma V_\phi^\pi(1)= \frac{1}{1-\gamma},\\
        \forall s \in \{2,3,\cdots,S\}:\qquad V_\phi^\pi(s) \overset{(ii)}{=} 0 + \gamma V_\phi^\pi(1) = \frac{\gamma}{1-\gamma},
    \end{gather}
\end{subequations}
where (i) and (ii) are according to the fact that the transitions defined over states $s\geq 1$ in (\ref{Eq:CVaR_transition_lower_bound}) give only one possible next state 1, and by the definition of the uncertainty set in (\ref{CVaR uncertainty set}), there exists only one transition kernel within the uncertainty set, which is the kernel itself.

The value function at state $s = 0$ with policy $\pi$ satisfies
\begin{align}
\label{V(0) for any policy}
        V_\phi^\pi(0) &= \mathbb{E}_{a\sim \pi(\cdot|0)}[r(0,a)] + \gamma \sum_{a\in \mathcal{A}}\pi(a|0)\inf_{\varient{\vect{P}} \in \mathcal{U}^\tau({\vect{P}_{0,a}^\phi)}} \varient{\vect{P}} \vect{V}_{\phi}^{\pi} \nn\\
        & = \gamma \sum_{a\in \mathcal{A}}\pi(a|0) \inf_{\varient{\vect{P}} \in \mathcal{U}^\tau({\vect{P}_{0,a}^\phi)}} \varient{\vect{P}} \vect{V}_{\phi}^{\pi} \nn \\
        & = \gamma \left(\pi(\phi|0)\inf_{\varient{\vect{P}} \in \mathcal{U}^\tau({\vect{P}_{0,\phi}^\phi)}}\varient{\vect{P}}V_{\phi}^\pi + 
        \left(1-\pi(\phi|0)\right)\inf_{\varient{\vect{P}} \in \mathcal{U}^\tau({\vect{P}_{0,1-\phi}^\phi)}}\varient{\vect{P}}V_{\phi}^\pi \right) \nn \\
        & \overset{(i)}{=} \gamma \left(\pi(\phi|0)\left(V_{\phi}^{\pi}(1)\underline{p} + V_{\phi}^{\pi}(0)(1-\underline{p})\right)+ 
        \left(1-\pi(\phi|0)\right)(V_{\phi}^{\pi}(1)\underline{q} + V_{\phi}^{\pi}(0)(1-\underline{q}))\right) \nn \\
        & = \frac{\gamma z_\phi^\pi }{1-\gamma( 1 - z_\phi^\pi)}V_{\phi}^{\pi}(1) \nn \\
        & = \frac{\gamma z_\phi^\pi }{(1-\gamma)(1-\gamma( 1 - z_\phi^\pi))},
\end{align}
where (i) follows from (\ref{Eq:smallest_transition_p}), (\ref{Eq:smallest_transition_q}) and (\ref{Eq:worst p&q}).
Note that $V_\phi^\pi(0)$ is increasing in $z_{\phi}^\pi$ and $z_{\phi}^\pi$ is upper bounded by $\underline{p}$ and reaches the upper bound when $\pi(\phi|0) = 1$. Taking $z_{\phi}^\pi = \underline{p} $ in (\ref{V(0) for any policy}), we get the optimal value function
\begin{equation}
    V_\phi^*(0) = \frac{\gamma \underline{p}}{(1-\gamma)(1-\gamma(1-\underline{p}))}.
\end{equation}
\end{document}